\documentclass[11pt]{article}
\usepackage{setspace}
\usepackage{amsmath,amssymb}
\usepackage{amsthm}
\usepackage{algorithm}
\usepackage[noend]{algorithmic}
\usepackage{url}
\usepackage{fullpage}
\usepackage{makeidx}
\usepackage{enumerate}
\usepackage[top=1in, bottom=1in, left=1in, right=1in]{geometry}
\usepackage{graphicx,float,psfrag,epsfig,caption,subcaption}
\usepackage{epstopdf}
\usepackage{color}

\usepackage[mathscr]{euscript}

\DeclareSymbolFont{rsfs}{U}{rsfs}{m}{n}
\DeclareSymbolFontAlphabet{\mathscrsfs}{rsfs}

\numberwithin{equation}{section}

\newtheoremstyle{myexample} 
    {\topsep}                    
    {\topsep}                    
    {\rm }                   
    {}                           
    {\bf }                   
    {.}                          
    {.5em}                       
    {}  

\newtheoremstyle{myremark} 
    {\topsep}                    
    {\topsep}                    
    {\rm}                        
    {}                           
    {\bf}                        
    {.}                          
    {.5em}                       
    {}  

\newtheorem{claim}{Claim}[section]
\newtheorem{lemma}[claim]{Lemma}

\newtheorem*{assumption}{Assumption}

\newtheorem{theorem}{Theorem}
\newtheorem{proposition}[claim]{Proposition}

\theoremstyle{myremark}
\newtheorem{remark}{Remark}[section]

\theoremstyle{myremark}

\theoremstyle{myexample}
\newtheorem{example}[remark]{Example}

\def\<{\langle}
\def\>{\rangle}

\def\eps{{\varepsilon}}

\def\id{{\rm I}}
\def\sT{{\sf T}}

\def\conv{{\rm conv}}

\def\one{\mathbf{1}}

\def\normal{{\sf N}}

\def\cS{{\cal S}}

\def\bbeta{\bar{\beta}}

\def\bq{\boldsymbol{q}}

\def\one{{\bf 1}}

\def\bG{{\bf G}}

\def\brho{\overline{\rho}}

\def\conv{{\rm{conv}}}
\def\aff{{\rm{aff}}}
\def\bzero{{\boldsymbol 0}}

\def\reals{\mathbb{R}}

\def\bB{{\boldsymbol B}}
\def\bM{{\boldsymbol M}}
\def\bE{{\boldsymbol E}}
\def\by{{\boldsymbol y}}

\def\bQ{{\boldsymbol Q}}

\def\Id{{\bf I}}

\def\bG{{\boldsymbol G}}
\def\bH{{\boldsymbol H}}

\def\bPi{{\boldsymbol \Pi}}
\def\be{{\boldsymbol e}}
\def\bpi{{\boldsymbol \pi}}

\def\bx{{\boldsymbol x}}
\def\bz{{\boldsymbol z}}
\def\ba{{\boldsymbol a}}

\def\bX{{\boldsymbol X}}
\def\bY{{\boldsymbol Y}}
\def\bZ{{\boldsymbol Z}}
\def\bP{{\boldsymbol P}}
\def\bv{{\boldsymbol v}}
\def\bu{{\boldsymbol u}}
\def\bw{{\boldsymbol w}}

\def\bW{{\boldsymbol W}}
\def\hbH{\widehat{\boldsymbol H}}
\def\tbH{\widetilde{\boldsymbol H}}

\def\bU{{\boldsymbol U}}
\def\bV{{\boldsymbol V}}
\def\bR{{\boldsymbol R}}
\def\bA{{\boldsymbol A}}
\def\bB{{\boldsymbol B}}
\def\bC{{\boldsymbol C}}
\def\balpha{{\boldsymbol \alpha}}
\def\bbeta{{\boldsymbol \beta}}
\def\brho{{\boldsymbol \rho}}
\def\blambda{{\boldsymbol \lambda}}

\def\ext{{\rm {ext}}}
\def\bSigma{{\boldsymbol \Sigma}}

\def\bh{{\boldsymbol h}}
\def\hbh{\hat{\boldsymbol h}}

\def\Ball{{\sf B}}

\def\Treg[#1]{T^{{\rm reg},#1}}
\def\GW[#1]{{\rm GW}(#1)}
\def\MGW[#1]{{\rm MGW}(#1)}

\def\cuD{\mathscrsfs{D}}

\def\cQ{{\mathcal Q}}
\def\bfone{{\boldsymbol 1}}
\def\cuL{\mathscrsfs{L}}
\def\cuR{\mathscrsfs{R}}
\def\balpha{{\boldsymbol \alpha}}
\def\Ind{{\rm\bf I}}

\title{Non-negative Matrix Factorization via Archetypal Analysis}
\author{Hamid Javadi\thanks{Department of Electrical Engineering, Stanford University}\;\;
\; and \;\; Andrea Montanari\thanks{Department of Electrical 
Engineering and Statistics, Stanford University}}

\begin{document}

\maketitle

\begin{abstract}
Given a collection of data points, non-negative matrix factorization (NMF) suggests to express  them as
convex combinations of a small set of `archetypes' with non-negative entries. This decomposition is unique only
if the true archetypes are non-negative and sufficiently sparse  (or the weights are sufficiently sparse), a 
regime that is captured by the separability condition and its generalizations. 

In this paper, we study an approach to NMF that can be traced back to the  work of Cutler and Breiman \cite{cutler1994archetypal} 
and does not require the data to be separable, while providing a generally unique decomposition. We optimize
the trade-off between two objectives: we minimize the distance of the data points from the convex envelope of the 
archetypes (which can be interpreted as an empirical risk), while minimizing the distance of the archetypes from the convex envelope of the data
(which can be interpreted as a data-dependent regularization). The archetypal analysis method of \cite{cutler1994archetypal} is recovered as the
limiting case in which the last term is given infinite weight. 

We introduce a `uniqueness condition' on the  data  which is necessary for exactly recovering the archetypes from noiseless data. We prove that, under uniqueness 
(plus additional regularity conditions on the geometry of the archetypes), our estimator is robust. While our approach requires solving a non-convex optimization
problem, we find that standard optimization methods succeed in  finding good solutions both for real and synthetic data.
\end{abstract}

\section{Introduction}

Given a set of data points $\bx_1,\bx_2,\cdots,\bx_n\in \reals^d$, it is often  useful to
represent them as convex combinations of a small set of vectors (the `archetypes' $\bh_1,\dots,\bh_{\ell}$):
\begin{align}
\bx_i \approx \sum_{\ell=1}^r w_{i,\ell}\bh_{\ell}\, ,\;\;\;\; w_{i,\ell}\ge   0, \;\; \sum_{\ell=1}^rw_{i,\ell} =1\, . \label{eq:Decomposition}
\end{align}
Decompositions of this type have wide ranging applications, from chemometrics \cite{paatero1994positive} to image processing
\cite{lee1999learning} and topic modeling \cite{xu2003document}. As an example, Figure \ref{fig:spectra-no-noise} displays the 
infrared reflection spectra\footnote{Data were  retrieved from the NIST Chemistry WebBook dataset \cite{nist}.} 
of four molecules (caffeine, sucrose, lactose and trioctanoin) for wavenumber between 
$1186\;  {\rm cm}^{-1}$ and $1530\; {\rm cm}^{-1}$. Each spectrum is a vector $\bh_{0,\ell}\in\reals^d$, with $d=87$ and 
$\ell\in \{1,\dots, 4\}$. If a mixture of these substances is analyzed, the resulting spectrum will be a convex combination of the 
spectra of the four analytes. This situation arises in hyperspectral imaging \cite{ma2014signal}, where a main focus is to 
estimate spatially  varying  proportions of a certain number of  analytes.
In order to mimic this setting, we generated $n=250$ synthetic random convex combinations $\bx_1,\dots,\bx_n\in\reals^d$
of the four spectra $\bh_{0,1}$, \dots, $\bh_{0,4}$, each containing two or more of these four  analytes, and tried to reconstruct the pure spectra from the $\bx_i$'s.
Each column in Figure \ref{fig:spectra-no-noise}  displays the reconstruction obtained using a different procedure.
We refer to Appendix \ref{app:Numerical} for further details.

\begin{figure}
  \phantom{A}\hspace{-2cm} \includegraphics[width=1.2\textwidth]{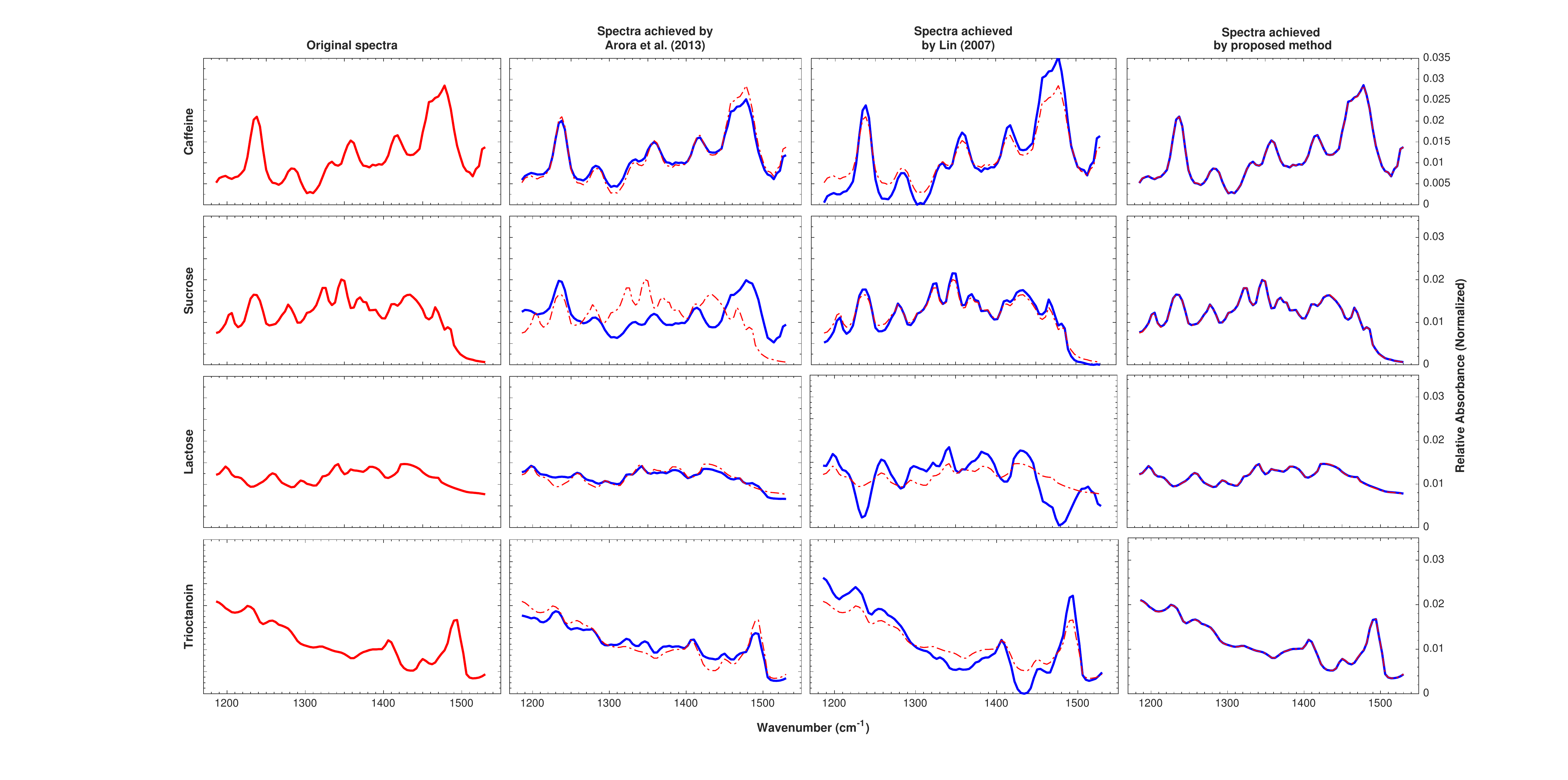}
\vspace{-1cm}

    \caption{Left column: Infrared reflection spectra of four molecules. Subsequent columns: Spectra estimated from $n=250$ spectra of mixtures of the
four original substances (synthetic data generated by taking random convex combinations of the pure spectra, see Appendix \ref{app:Numerical} for details). 
Each column reports the results obtained
with a different estimator: continuous blue lines correspond to the reconstructed spectra; dashed red lines correspond to the ground truth.}
    \label{fig:spectra-no-noise}
\end{figure}
Without further constraints, the decomposition (\ref{eq:Decomposition}) is dramatically underdetermined. 
Given a set of valid archetypes $\{\bh_{0,\ell}\}_{\ell\le r}$,  any set $\{\bh_{\ell}\}_{\ell\le r}$ whose convex hull contains the $\{\bh_{0,\ell}\}_{\ell\le r}$  
also satisfies Eq.~(\ref{eq:Decomposition}). For instance, we can set $\bh_\ell=\bh_{0,\ell}$ for $\ell\le r-1$,
and $\bh_r= (1+s)\bh_{0,r}-s\bh_{0,1}$ for any $s\ge 0$,  and obtain an equally good representation of the data $\{\bx_i\}_{i\le n}$.

How should we constrain the decomposition (\ref{eq:Decomposition}) in such a way that it is generally unique (up to permutations
of the $r$ archetypes)? Since the seminal work of Paatero and Tapper \cite{paatero1994positive,paatero1997least}, and of 
Lee and Seung \cite{lee1999learning,lee2001algorithms}, an overwhelming amount of work has addressed this question
by making the assumptions that the archetypes are componentwise non-negative $\bh_{\ell}\ge 0$. Among other applications
the non-negativity constraint is justified for chemometrics (reflection or absorption spectra are non-negative), and 
topic modeling (in this case archetypes correspond to topics, which are represented as probability distributions over words).
This formulation has become popular as non-negative matrix factorization (NMF).

Under the non-negativity constraint $\bh_{\ell}\ge 0$ the role of weights and archetypes becomes symmetric, and 
the decomposition (\ref{eq:Decomposition}) is unique provided that the archetypes or the weights are sufficiently sparse
(without loss of generality one can assume $\sum_{\ell=1}^rh_{\ell,i}=1$). This point was clarified by
Donoho and Stodden \cite{donoho2003does}, introduced a separability condition that ensure uniqueness. The non-negative archetypes
$\bh_1,\cdots,\bh_{r}$ are separable if, for each $\ell\in [r]$ there exists an index $i(\ell)\in[d]$ such that $(\bh_{\ell})_{i(\ell)}=1$,
and $(\bh_{\ell'})_{i(\ell)}=0$ for all $\ell'\neq \ell$. If we exchange the roles of  weights $\{w_{i,\ell}\}$ and archetypes $\{h_{\ell,i}\}$,
separability requires that $\ell\in [r]$ there exists an index $i(\ell)\in[n]$ such that $w_{i(\ell),\ell}=1$,
and $w_{i(\ell),\ell'}=0$ for all $\ell'\neq \ell$
This condition has a simple geometric interpretation: the data are separable
if for each archetype $\bh_{\ell}$ there is at least one data point $\bx_i$ such that $\bx_i=\bh_{\ell}$.
A copious literature has developed algorithms for non-negative matrix factorization under 
separability condition or its generalizations \cite{donoho2003does,arora2012computing,recht2012factoring,arora2013practical,ge2015intersecting}.  

Of course this line of work has a drawback: in practice we do not know
whether the data are separable. (We refer to the Section \ref{sec:Discussion} for a comparison with 
\cite{ge2015intersecting}, which relaxes the separability assumption.) 
Further, there are many cases in which the archetypes 
$\bh_1,\dots,\bh_{\ell}$ are not necessarily non-negative. For instance, in spike sorting, the data are measurements of neural activity ad the 
archetypes correspond to waveforms associated to different neurons \cite{roux2009adaptive}. In other applications the archetypes
$\bh_{\ell}$ are non-negative, but --in order to reduce complexity-- the data $\{\bx_i\}_{i\le n}$  are replaced by a random low-dimensional projection \cite{kim2008fast,wang2010efficient}.
The projected archetypes loose the non-negativity property.
 Finally, the  decomposition (\ref{eq:Decomposition})  is generally non-unique, even under the constraint $\bh_{\ell}\ge 0$. 
This is illustrated, again, in Figure~\ref{fig:spectra-no-noise}: all the spectra are strictly positive, and hence we can find archetypes $\bh_1,\dots,\bh_4$ that are still non-negative
and whose convex envelope contains $\bh_{0,1},\dots,\bh_{0,4}$.

Since NMF is underdetermined, standard methods fail in such applications, as illustrated in Figure~\ref{fig:spectra-no-noise} . 
We represent the data as a matrix $\bX\in\reals^{n\times d}$
whose $i$-th row is the vector $\bx_i$, the weights by a matrix $\bW = (w_{i,\ell})_{i\le n, \ell\le r}\in\reals^{n\times d}$ and the prototypes by a 
matrix $\bH= (h_{\ell,j})_{\ell\le r, j\le d}\in\reals^{r\times d}$. The third column of Figure~\ref{fig:spectra-no-noise} 
uses a projected gradient algorithm from \cite{lin2007projected} to solve the problem
\begin{align}
\mbox{minimize}&\;\;\;\;\;\; \|\bX-\bW\bH\|_F^2\, ,\label{eq:StandardNMF}\\
\mbox{subject to}& \;\;\;\;\;\;  \bW\ge 0, \;\; \bH\ge 0\, .\nonumber
\end{align}
Empirically, projected gradient converges to a point with very small fitting error $\|\bX-\bW\bH\|_F^2$, but the reconstructed spectra (rows of $\bH$)
are inaccurate. The second column in the same figure shows the spectra reconstructed using an algorithm from \cite{arora2013practical}, that assumes separability:
as expected, the reconstruction is not accurate.

In a less widely known paper, Cutler and Breiman  \cite{cutler1994archetypal} addressed the same problem using what they call `archetypal analysis.'
Archetypal analysis presents two important differences with respect to standard NMF: $(1)$ The archetypes $\bh_{\ell}$ are not necessarily required to be non-negative
(although this constraint can be easily incorporated);
$(2)$ The under-determination of the decomposition (\ref{eq:Decomposition}) is addressed by requiring that the archetypes belong to the convex hull
of the data points: $\bh_{\ell}\in \conv(\{\bx_{i}\}_{i\le n})$. 

In applications the condition $\bh_{\ell}\in \conv(\{\bx_{i}\}_{i\le  n})$  is too strict. 
This paper builds on the ideas of \cite{cutler1994archetypal} to propose a formulation of NMF that is
uniquely defined (barring degenerate cases) and provides a useful notion of optimality. 
In particular, we present the following contributions.

\vspace{0.1cm}

\noindent{\bf Archetypal reconstruction.} We propose to reconstruct the archetypes $\bh_1,\dots,\bh_{r}$ by optimizing a combination of
two objectives. On one hand, we minimize the error in the decomposition (\ref{eq:Decomposition}). This amounts to minimizing the distance between 
the data points and the convex hull of the archetypes. On the other hand, we minimize the distance of the archetypes from the convex hull of data 
points. This relaxes the original condition imposed in \cite{cutler1994archetypal} which required the archetypes to lie in
$\conv(\{\bx_{i}\})$, and allows to treat non-separable data. 

\vspace{0.1cm}

\noindent{\bf Robustness guarantee.}  We next assume that that the decomposition (\ref{eq:Decomposition}) approximately hold
for some `true' archetypes $\bh^0_{\ell}$ and weights $w_{i,\ell}^0$, namely  $\bx_i = \bx_i^0+\bz_i$, where 
$\bx^0_i=\sum_{\ell=1}^r w^0_{i,\ell}\bh^0_{\ell}$ and $\bz_i$ captures unexplained effects. We introduce a `uniqueness condition' on the  
data $\{\bx^0_i\}_{i\le n}$ which is necessary for exactly recovering the archetypes from the noiseless data. 
We prove that, under uniqueness (plus additional regularity conditions on the geometry of
the archetypes), our estimator is robust. Namely it outputs archetypes $\{\hbh_\ell\}_{\ell\le r}$ whose distance from the true ones $\{\bh^0_{\ell}\}_{\ell\le r}$
(in a suitable metric)  is controlled by $\sup_{i\le n}\|\bz_i\|_2$. 

\vspace{0.1cm}

\noindent{\bf Algorithms.} Our approach reconstructs the archetypes $\bh_1,\dots,\bh_r$ by minimizing a non-convex risk function $\cuR_{\lambda}(\bH)$.  We propose 
three descent algorithms that appear to perform well on realistic instances of the problem. 
In particular, Section \ref{sec:Algorithm} introduces a proximal alternating linearized minimization algorithm (PALM) that is guaranteed to converge to critical points
of the risk function. Appendix \ref{app:algo} discusses two alternative approaches. One possible explanation for the success of such descent algorithms is that reasonably 
good initializations can be constructed using spectral methods, or approximating the data as separable, cf. Section \ref{sec:Initialization}. We defer a study of global convergence of
this two-stages approach to future work.

\section{An archetypal reconstruction approach}

\begin{figure}
\begin{center}
\includegraphics[width=0.4\textwidth]{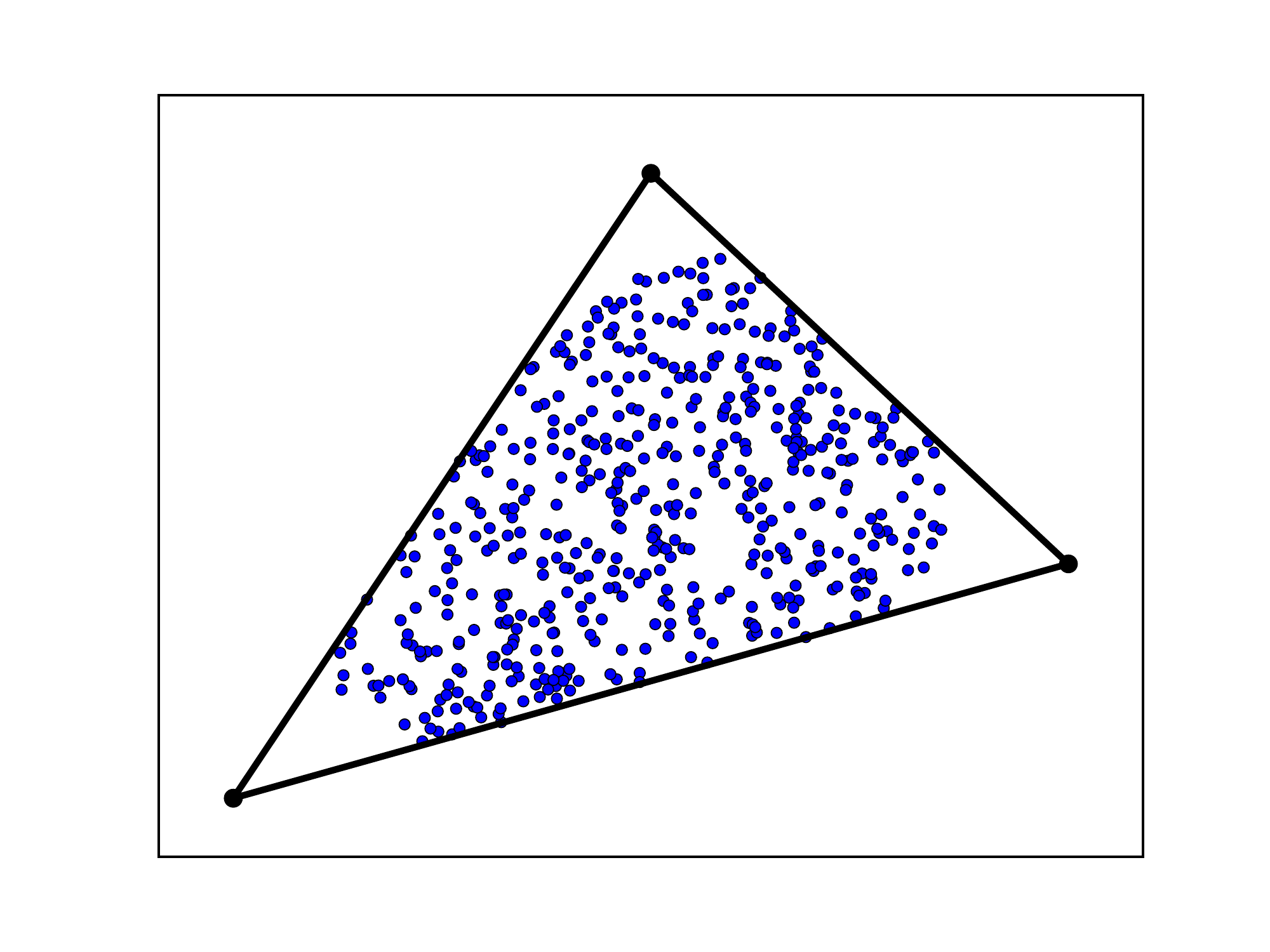}\hspace{-0.5cm}
\includegraphics[width=0.4\textwidth]{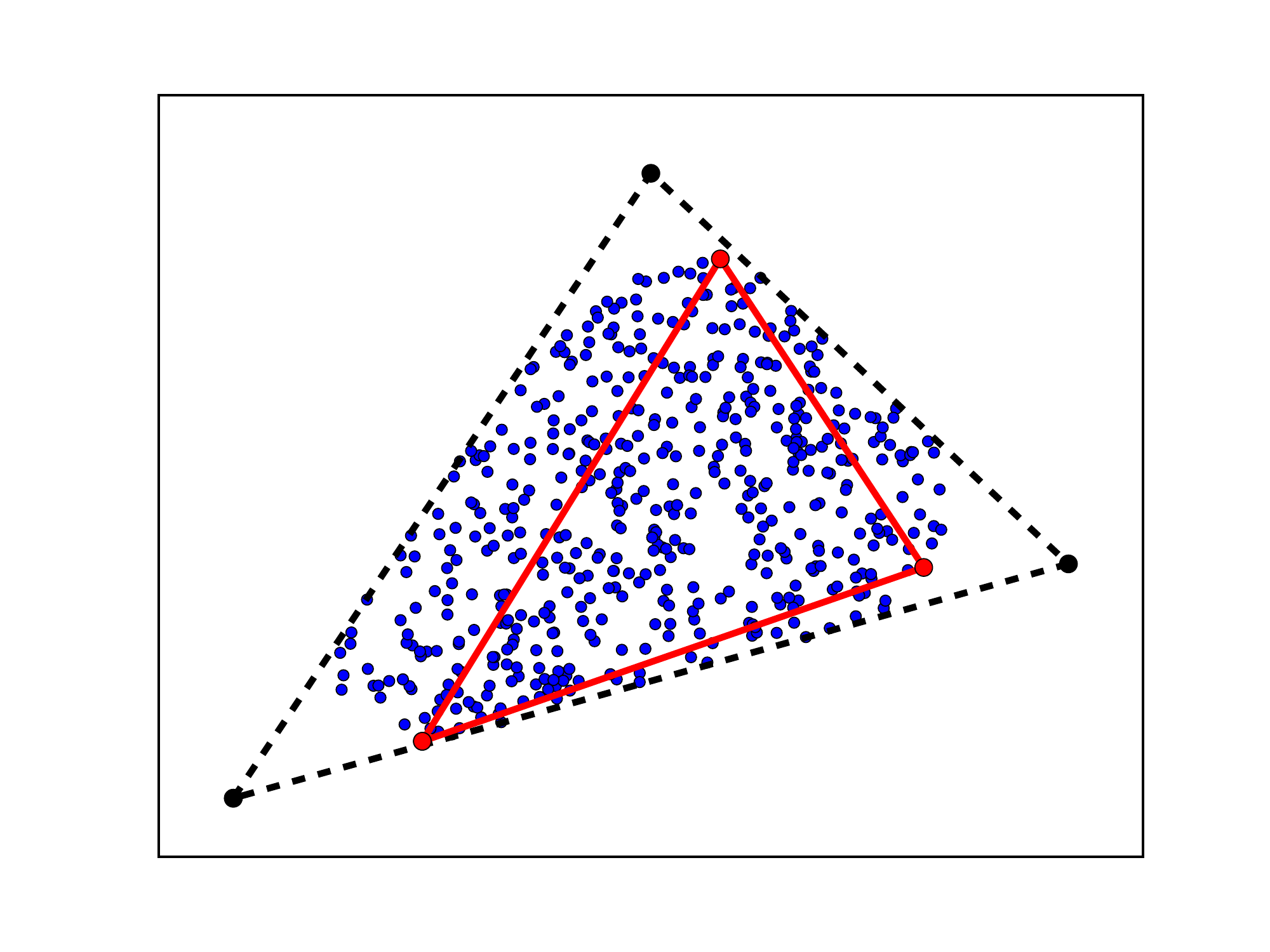}\\
\includegraphics[width=0.4\textwidth]{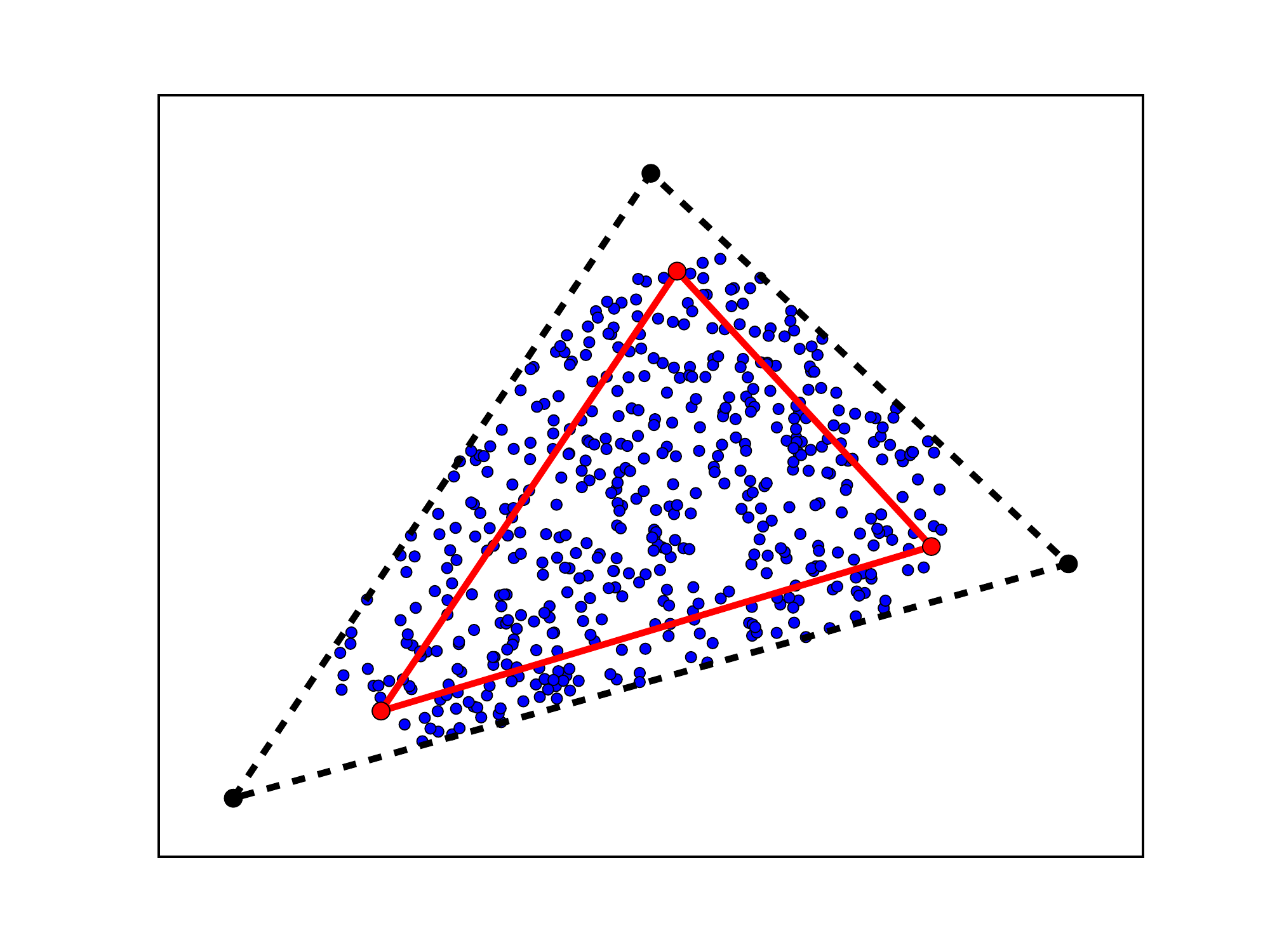}\hspace{-0.5cm}
\includegraphics[width=0.4\textwidth]{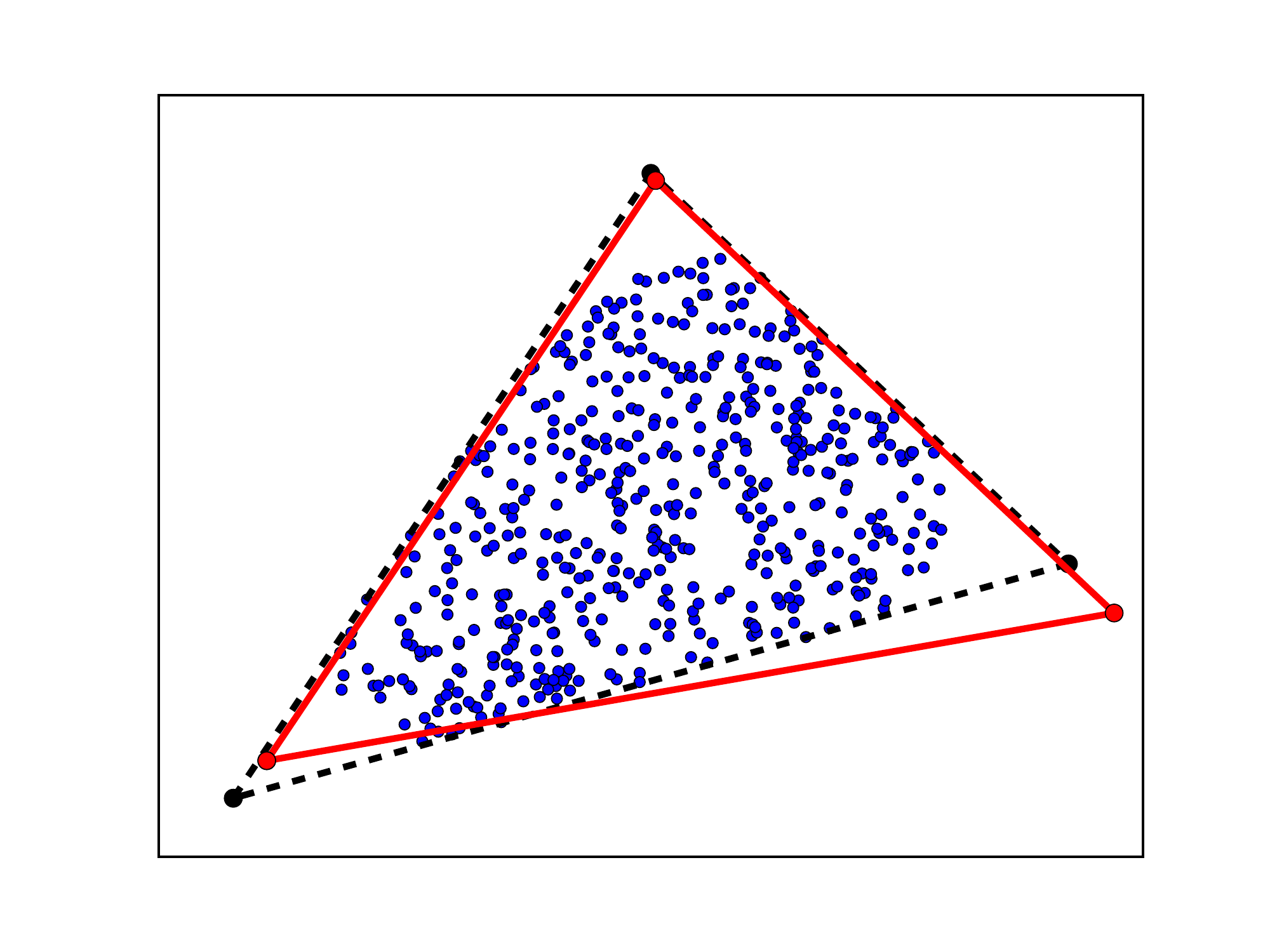}
\end{center}
    \caption{Toy example of archetype reconstruction. Top left: data points (blue) are generated as random linear combinations
of $r=3$ archetypes in $d=2$ dimensions (red, see Appendix
\ref{app:Numerical} for details). Top right: Initialization using the algorithm of \cite{arora2013practical}.  
Bottom left: Output of the alternate minimization algorithm of
\cite{cutler1994archetypal} with initialization form the previous
frame. Bottom right: Alternate minimization algorithm to compute
the estimator (\ref{eq:Lagrangian}), with $\lambda = 0.0166$.}
    \label{fig:example}
\end{figure}

Let $\cQ\subseteq \reals^d$ be a convex set and $D:\cQ\times \cQ\to\reals$, $(\bx,\by)\mapsto D(\bx;\by)$ a loss function on $\cQ$. For a point $\bu\in\cQ$, and a matrix $\bV\in\reals^{m\times d}$,
with rows $\bv_1,\dots,\bv_m\in\cQ$, we let
\begin{align}
\cuD(\bu;\bV) &\equiv \min \Big\{  D\big(\bu ; \bV^{\sT}\bpi\big)\, :\;\;\;\; \bpi  \in \Delta^m\, \Big\} \, ,\\
\Delta^m &\equiv \big\{\bx\in\reals^m_{\ge 0}:\;\; \<\bx,\bfone\> = 1\big\}\, .
\end{align}
In other words, denoting by $\conv(\bV) = \conv(\{\bv_1,\dots,\bv_m\})$ the convex hull of the rows of matrix $\bV$, 
$\cuD(\bu;\bV)$ is the minimum loss between $\bx$ and any point in  $\conv(\bV)$.
If $\bU\in\reals^{k\times d}$ is a matrix with rows $\bu_{1},\dots, \bu_{k}\in\cQ$, we generalize this definition by letting
\begin{align}
\cuD(\bU;\bV) \equiv \sum_{\ell=1}^k \cuD(\bu_{\ell};\bV)\, .
\end{align}
While this definition makes sense more generally, we have in mind two specific examples in which $D(\bx;\by)$ is actually separately convex in its arguments $\bx$ and $\by$.
(Most of our results will concern the first example.)

\begin{example}[Square loss]
In this case $\cQ = \reals^d$, and $D(\bx;\by) = \|\bx-\by\|_2^2$. This is the case originally studied by Cutler and Breiman \cite{cutler1994archetypal}.
\end{example}

\begin{example}[KL divergence]
We take $\cQ = \Delta^d$, the $d$-dimensional simplex, and $D(\bx;\by)$ to be the Kullback-Leibler
divergence between probability distributions $\bx$ and $\by$, namely $D(\bx;\by)  \equiv \sum_{i=1}^dx_i\log(x_i/y_i)$. 
\end{example}

Given data $\bx_1,\dots,\bx_n$  organized in the matrix $\bX\in\reals^{n\times d}$, we estimate the archetypes by solving the
problem\footnote{This problem can have multiple global minima if $\lambda=0$ or in degenerate settings. One minimizer is selected arbitrarily when this happens.}
\begin{align}
\hbH_{\lambda}\in \arg\min\Big\{\cuD(\bX;\bH)+\lambda\, \cuD(\bH;\bX): \;\; \bH\in\cQ^{r}\Big\}\, ,\label{eq:Lagrangian}
\end{align}
where we denote by $\cQ^r$ the set of matrices $\bH\in\reals^{r\times d}$ with rows $\bh_1,\dots,\bh_r\in\cQ$.
A few values of $\lambda$ are of  special significance. If we set
$\lambda =0$, and $\cQ =\Delta^d$, we recover the standard NMF objective (\ref{eq:StandardNMF}),
with a more general distance function $D(\,\cdot\,,\,\cdot\,)$. As pointed out above, in general this objective has no unique minimum.
If we let $\lambda\to 0+$ after the minimum is evaluated,
$\hbH_{\lambda}$ converges to the minimizer of $\cuD(\bX;\bH)$ which is the `closest'
to the convex envelope of the data $\conv(\bX)$ (in the sense of minimizing $\cuD(\bH;\bX)$). 
Finally as $\lambda\to\infty$, the archetypes $\bh_\ell$ are forced to lie in $\conv(\bX)$ and
hence we recover the method of \cite{cutler1994archetypal}.

Figure \ref{fig:example} illustrates  the advantages of the estimator (\ref{eq:Lagrangian}) on a small synthetic example, with $d=2$, $r=3$, $n=500$:
in this case the data are non separable. 
We first use the successive projections algorithm of \cite{arora2013practical} (that is designed to deal with separable data) in order to estimate the archetypes.
As expected, the reconstruction is not accurate because this algorithm assumes separability and hence estimates the archetypes with a subset of
the data points. We then use these estimates as initialization in the alternate minimization algorithm of \cite{cutler1994archetypal},
which optimizes the objective (\ref{eq:Lagrangian}) with $\lambda=\infty$. The estimates improve but not substantially: they are still constrained to lie in $\conv(\bX)$.
A significant improvement is obtained by setting $\lambda$ to a small value. We (approximately) minimize the cost function
(\ref{eq:Lagrangian}) by generalizing the alternate minimization algorithm, cf. Section \ref{sec:Algorithm}.
The optimal archetypes are no longer constrained to $\conv(\bX)$, and provide a better estimate of the true archetypes.
The last column in Figure \ref{fig:spectra-no-noise} uses the same estimator, and approximately solves problem (\ref{eq:Lagrangian}) by  gradient descent algorithm.

In our analysis we will consider a slightly different formulation in which the Lagrangian of Eq.~(\ref{eq:Lagrangian}) is
replaced by a hard constraint:
\begin{align}
\mbox{minimize}&\;\;\;\;\; \cuD(\bH;\bX)\, ,\label{eq:HardNoise}\\
\mbox{subject to }&\;\;\;\;\; \cuD(\bx_{i};\bH)\le \delta^2\;\;\;\mbox{ for all } i\in\{1,\dots,n\}\, .\nonumber
\end{align}
We will use this version in the analysis presented in the next section, and denote the corresponding estimator by $\hbH$.

\section{Robustness}

In order to analyze the robustness properties of estimator $\hbH$, we assume that there exists an approximate factorization 
\begin{align}
\bX = \bW_0\bH_0+ \bZ \, ,\label{eq:AssDecomposition}
\end{align}
where $\bW_0\in\reals^{n\times r}$ is a matrix of weights (with rows $\bw_{0,i}\in\Delta^r$), $\bH_0\in\reals^{r\times d}$ is a matrix of archetypes 
(with rows 
$\bh_{0,\ell}$),  and we set $\bX_0 = \bW_0\bH_0$. The deviation $\bZ$ is  arbitrary, with rows $\bz_i$ satisfying $\max_{i\le n}\|\bz_{i}\|_2\le \delta$.
We will assume throughout $r$ to be known.

We will quantify estimation error by the sum of distances between the true archetypes and the closest estimated archetypes
\begin{align}
\cuL(\bH_0,\hbH)\equiv \sum_{\ell=1}^r\min_{\ell'\le r} D(\bh_{0,\ell},\hbh_{\ell'})\, .
\end{align}
In words, if $\cuL(\bH_0,\hbH)$ is small, then for each true archetype $\bh_{0,\ell}$ there exists an estimated archetype $\hbh_{\ell'}$
that is close to it in $D$-loss. Unless two or more of the true archetypes are close to each other, this means that there is a one-to-one
correspondence between estimated archetypes and true archetypes, with small errors.
 
\begin{assumption}[Uniqueness]
We say that the factorization $\bX_0=\bW_0\bH_0$ satisfies uniqueness with parameter $\alpha>0$ (equivalently, is \emph{$\alpha$-unique}) if
for all $\bH\in\cQ^{r}$ with $\conv(\bX_0)\subseteq \conv(\bH)$, we have
\begin{align}
\cuD(\bH,\bX_0)^{1/2}\ge \cuD(\bH_0,\bX_0)^{1/2}+ \alpha\,\big\{\cuD(\bH,\bH_0)^{1/2}+\cuD(\bH_0,\bH)^{1/2}\big\}\, . \label{eq:UniquenessAssumption}
\end{align}
\end{assumption} 
The rationale for this assumption is quite clear. Assume that the data lie in the convex hull of the true archetypes $\bH_0$,
and hence Eq.~(\ref{eq:AssDecomposition}) holds without error term $\bZ=0$, i.e. $\bX=\bX_0$. We reconstruct the archetypes by 
demanding $\conv(\bX_0)\subseteq \conv(\bH)$: any such $\bH$ is a plausible explanation of the data. 
In order to make the problem well specified, we define $\bH_0$ to be the matrix of archetypes that are the closest to $\bX_0$,
and hence $\cuD(\bH,\bX_0)\ge \cuD(\bH_0,\bX_0)$ for all $\bH$. In order for the reconstruction to be unique (and hence for the problem to be identifiable)
we need to assume $\cuD(\bH,\bX_0) >\cuD(\bH_0,\bX_0)$ strictly for $\bH\neq \bH_0$. The uniqueness assumption
provides a quantitative version of this condition.

\begin{remark}
Given $\bX_0$, $\bH_0$, the best constant $\alpha$ such that Eq.~(\ref{eq:UniquenessAssumption}) holds for all $\bH$
is a geometric property that depend on $\bX_0$ only through $\conv(\bX_0)$.
In particular, if $\bX_0= \bW_0\bH_0$ is a separable factorization, then it satisfies uniqueness with parameter $\alpha=1$. Indeed in this 
case $\conv(\bH_0) = \conv(\bX_0)$, whence $\cuD(\bH,\bX_0) = \cuD(\bH,\bH_0)$ and $\cuD(\bH_0,\bX_0)=\cuD(\bH_0,\bH) = 0$.

It is further possible to show that $\alpha\in [0,1]$ for all $\bH_0,\bX_0$.  Indeed, we took $\bH_0$ to be the matrix of archetypes
that are closest to $\bX_0$. In other words, $\cuD(\bH,\bX_0)\ge \cuD(\bH_0,\bX_0)$ and hence, $\alpha \geq 0$. In addition,
since $\conv(\bX_0) \subseteq \conv(\bH_0)$, for $\bh_i$ an arbitrary row of $\bH$ we have
\begin{align}
\cuD(\bh_i,\bX_0) \leq \cuD(\bh_i,\bH_0).
\end{align}
Hence, $\cuD(\bH,\bX_0) \leq \cuD(\bH,\bH_0)$ and therefore 
\begin{align}
\cuD(\bH,\bX_0)^{1/2}\le \cuD(\bH_0,\bX_0)^{1/2}+\big\{\cuD(\bH,\bH_0)^{1/2}+\cuD(\bH_0,\bH)^{1/2}\big\}\, .
\end{align}
Thus, $\alpha \leq 1$.
\end{remark}

We say that the convex hull $\conv(\bX_0)$ has \emph{internal radius} (at least) $\mu$ if it contains an $r-1$-dimensional ball
of radius $\mu$, i.e. if there exists $\bz_0 \in \reals^d$, $\bU\in \reals^{d\times (r-1)}$, with $\bU^\sT\bU = \Id_d$ , such that
$\bz_0 + \bU\Ball_{r-1}(\mu)\subseteq \conv(\bX_0)$. We further denote by $\kappa(\bM)$ the condition number of matrix $\bM$.
\begin{theorem}\label{thm:Robust}
Assume $\bX = \bW_0\bH_0+ \bZ$ where the factorization $\bX_0=\bW_0\bH_0$ satisfies the uniqueness assumption with parameter
$\alpha>0$, and that $\conv(\bX_0)$ has internal radius $\mu>0$.
Consider the estimator $\hbH$ defined by Eq.~(\ref{eq:HardNoise}), with $D(\bx,\by) = \|\bx-\by\|_2^2$ (square loss) and $\delta = \max_{i\le n} \|\bZ_{i,\cdot}\|_2$.
 If 
\begin{align}
\max_{i\le n} \|\bZ_{i,\cdot}\|_2\le \frac{\alpha\mu}{30 r^{3/2}}\, ,
\end{align}
then, we have
\begin{align}
\cuL(\bH_0,\hbH)\le \frac{C_*^2\,  r^{5}}{\alpha^2} \max_{i\le n} \|\bZ_{i,\cdot}\|^2_2\, ,
\end{align}
where $C_*$ is a coefficient that depends uniquely on the geometry of $\bH_0$, $\bX_0$, namely $C_* = 120(\sigma_{\rm max}(\bH_0)/\mu)\cdot
\max(1, \kappa(\bH_0)/\sqrt{r})$.
\end{theorem}

\section{Algorithms}
\label{sec:Algorithm}

While our main focus is on structural properties of non-negative matrix factorization, 
we provide evidence that the optimization problem we defined can be solved in practical scenarios. A more detailed
study is left to future work.

From a computational point of view, the Lagrangian formulation \eqref{eq:Lagrangian} is more appealing. For the sake of simplicity, we  denote the regularized risk by
\begin{align}
\cuR_{\lambda}(\bH) \equiv \cuD(\bX;\bH)+\lambda\, \cuD(\bH;\bX)\, ,\label{eq:RLagrangian}
\end{align}
and leave implicit the dependence on the data $\bX$. 
Notice that this function is non-convex and indeed has multiple global minima: in particular, permuting the rows of a minimizer $\bH$ yields other minimizers.
We will describe two greedy optimization algorithms: one based on gradient descent, and one on alternating minimization, which 
generalizes the algorithm of \cite{cutler1994archetypal}. In both cases it is helpful to use a good initialization: two initialization methods are introduced in the next section.

\subsection{Initialization}
\label{sec:Initialization}

We experimented with two initialization methods, described below. 

\vspace{0.1cm}

\noindent{\emph (1) \emph{Spectral initialization.}} Under the assumption that the archetypes $\{\bh_{0,\ell}\}_{\ell\le r}$ are linearly independent (and for non-degenerate weights $\bW$),
the `noiseless' matrix $\bX_0$ has rank exactly $r$. This motivates the following approach. We compute the singular value decomposition $\bX =\sum_{i=1}^{n\wedge d} \sigma_i\bu_i\bv_i^{\sT}$,
$\sigma_1\ge \sigma_2\ge \dots\ge \sigma_{n\wedge d}$,
and initialize $\hbH$ as the matrix $\hbH^{(0)}$ with rows $\hbh^{(0)}_1=\bv_1,\dots,\hbh^{(0)}_r=\bv_r$.

\vspace{0.1cm}

\noindent{\emph (2) \emph{Successive projections initialization.}}  We initialize $\hbH^{(0)}$ by choosing archetypes $\{\hbh^{(0)}_\ell\}_{1\le \ell\le r}$
that are a subset of the data $\{\bx_i\}_{1\le i\le n}$, selected as follows. The first archetype $\hbh^{(0)}_1$ is the data point which is farthest from 
the origin. For each subsequent archetype, we choose the point that is farthest from the affine subspace spanned by the previous ones.
\begin{center}
	\begin{tabular}{ll}
	\hline
	\vspace{-0.35cm}\\
	\multicolumn{2}{l}{ {\sc Archetype initialization algorithm}}\\
	\hline
	\vspace{-0.35cm}\\
	\multicolumn{2}{l}{ {\bf Input :}  Data $\{\bx_i\}_{i\le n}$,  $\bx_i\in\reals^d$; integer $r$;}\\
	\multicolumn{2}{l}{ {\bf Output :} Initial archetypes $\{\hbh_{\ell}^{(0)}\}_{1\le \ell\le r}$;}\\
        1: & Set $i(1) = \arg\max \{ D(\bx_{i};\bzero):\; i\le n\}$;\\
        2: & Set $\hbh^{(0)}_{1}= \bx_{i(1)}$;\\
	3: & For $\ell\in \{1,\dots, r\}$\\
        4: &\phantom{aa} Define $V_{\ell}\equiv \aff(\hbh_{1}^{(0)},\hbh_{2}^{(0)},\dots,\hbh_{\ell}^{(0)})$;\\
        5: &\phantom{aa} Set $i(\ell+1) = \arg\max \{\cuD(\bx_{i};V_{\ell})\, :\; i\le n\}$;\\
        6: &\phantom{aa} Set $\hbh^{(0)}_{\ell+1} = \bx_{i(\ell+1)}$;\\
	7: & End For;\\
	8: & Return $\{\hbh_{\ell}^{(0)}\}_{1\le \ell\le r}$l\\
	\vspace{-0.35cm}\\
	\hline
	\end{tabular}
\end{center}
This coincides with the successive projections algorithm  of \cite{araujo2001successive}, with the minor difference that $V_{\ell}$ is the affine
subspace spanned by the first $\ell$ vectors, instead of the linear subspace\footnote{The same modification is also used in \cite{arora2013practical}, but we do not apply the full 
algorithm of   this paper.} This method can be proved to return the exact archetypes if data are separable the archetypes are affine independent \cite{arora2013practical,gillis2014fast}.
When data are not separable it provides nevertheless a good initial assignment.

\subsection{Proximal alternating linearized minimization}
\label{sec:PALM}

The authors of  \cite{bolte2014proximal} develop a proximal alternating linearized minimization algorithm (PALM) to 
solve the problems of the form
\begin{align}\label{eq:palmform}
{\rm{minimize}}\quad\quad \Psi(\bx,\by) = f(\bx) + g(\by) + h(\bx,\by)
\end{align}
where $f:\reals^m\to (-\infty,+\infty]$ and $g:\reals^n\to (-\infty,\infty]$ are lower semicontinuous and $h\in C^1(\reals^{m}\times\reals^n)$. 
PALM is guaranteed to converge to critical points of the function $\Psi$ \cite{bolte2014proximal}.

We apply this algorithm to minimize the cost function (\ref{eq:RLagrangian}), with $D(\bx,\by)=\|\bx-\by\|_2^2$ which we write as
\begin{align}
\cuR_\lambda(\bH) =\min_{\bW} \Psi (\bH,\bW) = f(\bH) + g(\bW) + h(\bH,\bW)\, .
\end{align}
where,
\begin{align}
&f(\bH) = \lambda\,  \cuD(\bH,\bX),\\
&g(\bW) = \sum_{i=1}^n \Ind\left(\bw_i\in \Delta^r\right),\\
&h(\bH,\bW)  = \left\|\bX - \bW\bH\right\|_F^2.
\end{align}
In above equations $\bw_i$ are the rows of $\bW$ and the indicator function $\Ind(\bx\in \Delta^r)$ is equal to zero if 
$\bx \in \Delta^r$ and is equal to infinity otherwise.

By using this decomposition, the iterations of the PALM iteration reads
\begin{align}\label{eq:PALMITER1}
&\tbH^{k} = \bH^{k} - \frac{1}{\gamma_1^k}(\bW^k)^\sT\left(\bW^k\bH^k - \bX\right),\\
&\bH^{k+1} = \tbH^{k} - \frac{\lambda}{\lambda + \gamma_1^k}\left(\tbH^{k} - \bPi_{\conv(\bX)}\left(\tbH^{k}\right)\right),\\
&\bW^{k+1} = \bPi_{\Delta^r}\left(\bW^k - \frac{1}{\gamma_2^k}\left(\bW^k\bH^{k+1}-\bX\right)(\bH^{k+1})^\sT\right),
\label{eq:PALMITER3}
\end{align}
where $\gamma_1^k$, $\gamma_2^k$ are step sizes and, for $\bM\in\reals^{m_1\times m_2}$,  and $\cS\subseteq\reals^{m_2}$ a closed convex set,
$\bPi_{\cS}(\bM)$ is the matrix obtained by projecting the rows of $\bM$ 
onto the simplex $\cS$.
\begin{proposition}\label{propo:PALM}
Consider the risk (\ref{eq:RLagrangian}), with loss $D(\bx,\by) = \|\bx-\by\|_2^2$, and the corresponding cost function $\Psi(\bH,\bW)$.
If the step sizes are chosen such that $\gamma_1^k > \left\|\bW^{k^\sT}\bW^k\right\|_F$,  $\gamma_2^k > \max\left\{\left\|\bH^{k+1}\bH^{k+1^\sT}\right\|_F,\eps\right\}$ for some constant $\eps>0$,
then $(\bH^k,\bW^k)$ converges to a stationary point of the function $\Psi(\bH,\bW)$.
\end{proposition}
The proof of this statement is deferred to Appendix \ref{app:PALM}.

It is also useful to notice that the gradient of $\cuR_{\lambda}(\bH)$ can be computed explicitly (this can be useful to devise a stopping criterion).
\begin{proposition}\label{propo:Subdiff}
Consider the risk (\ref{eq:RLagrangian}), with loss $D(\bx,\by) = \|\bx-\by\|_2^2$,
and assume that the rows of $\bH$ are affine independent. Then, $\cuR_{\lambda}$ is differentiable at $\bH$ with gradient
\begin{align}
\nabla\cuR_{\lambda}(\bH) &=
2\sum_{i=1}^{n}\balpha^*_{i}\left(\bPi_{\conv(\bH)}(\bx_{i})-\bx_{i}\right)  +2\lambda\big(\bH - \bPi_{\conv(\bX)}(\bH)\big)\, ,\\
\balpha_{i}^{*} &= \arg\min_{\balpha \in \Delta^r} \left\| \bH^\sT \balpha - \bx_{i}^\sT\right\|_2\,.\label{eq:Alphai}
\end{align}
where we recall that $\bPi_{\conv(\bX)}(\bH)$ denotes the matrix with rows
$\bPi_{\conv(\bX)}(\bH_{1,\cdot}),\dots,\bPi_{\conv(\bX)}(\bH_{r,\cdot})$.
\end{proposition}
The proof of this proposition is given in Appendix \ref{app:Subdiff}.
Appendix \ref{app:algo} also discusses two alternative algorithms.

\begin{figure}
\phantom{A}\hspace{-2cm}
    \includegraphics[width=1.2\textwidth]{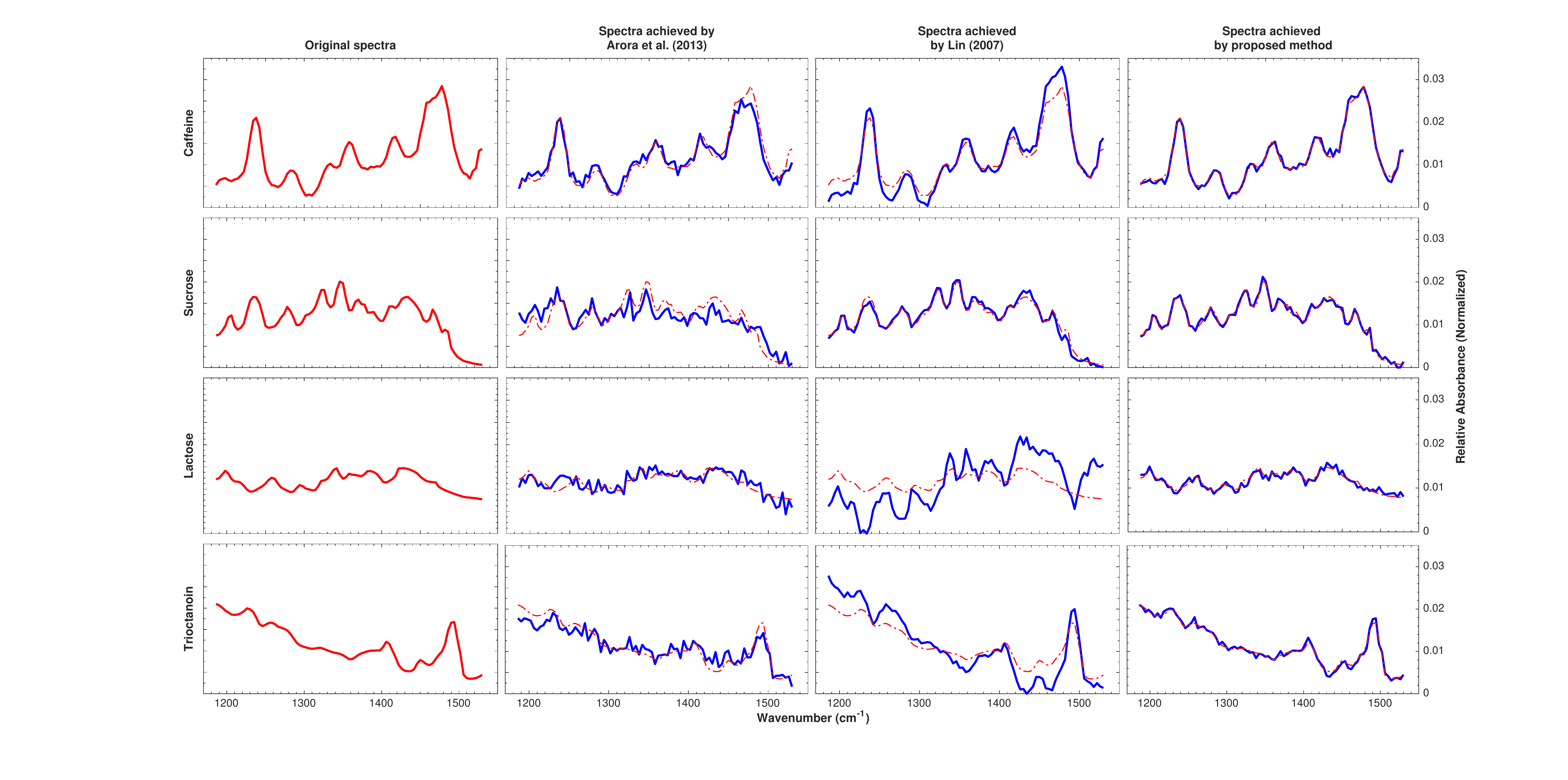}
     \vspace{-1cm}

    \caption{Reconstructing infrared spectra of four molecules, from noisy random convex combinations. Noise level $\sigma = 10^{-3}$.
Left column: original spectra. The other columns correspond to different reconstruction methods.}\label{fig:projgradsignals_lownoise}
\end{figure}
\begin{figure}
\phantom{A}\hspace{-2cm}
    \includegraphics[width=1.2\textwidth]{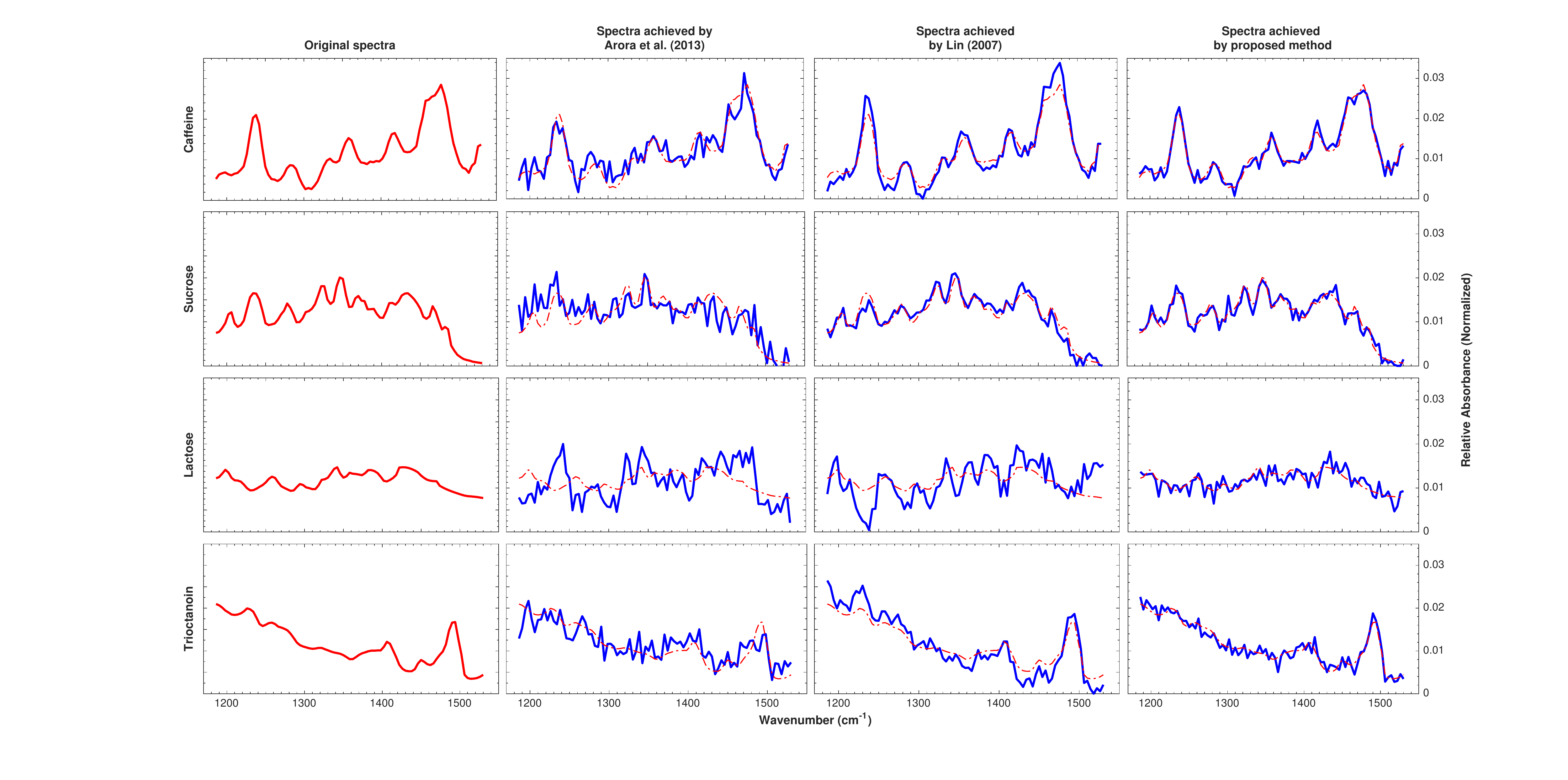}
     \vspace{-1cm}

    \caption{As in Figure \ref{fig:projgradsignals_lownoise}, with $\sigma = 2\cdot 10^{-3}$ (in blue).}\label{fig:oursignals_highnoise}
\end{figure}
\begin{figure}
 \phantom{A}\hspace{3cm}\includegraphics[width=0.65\textwidth]{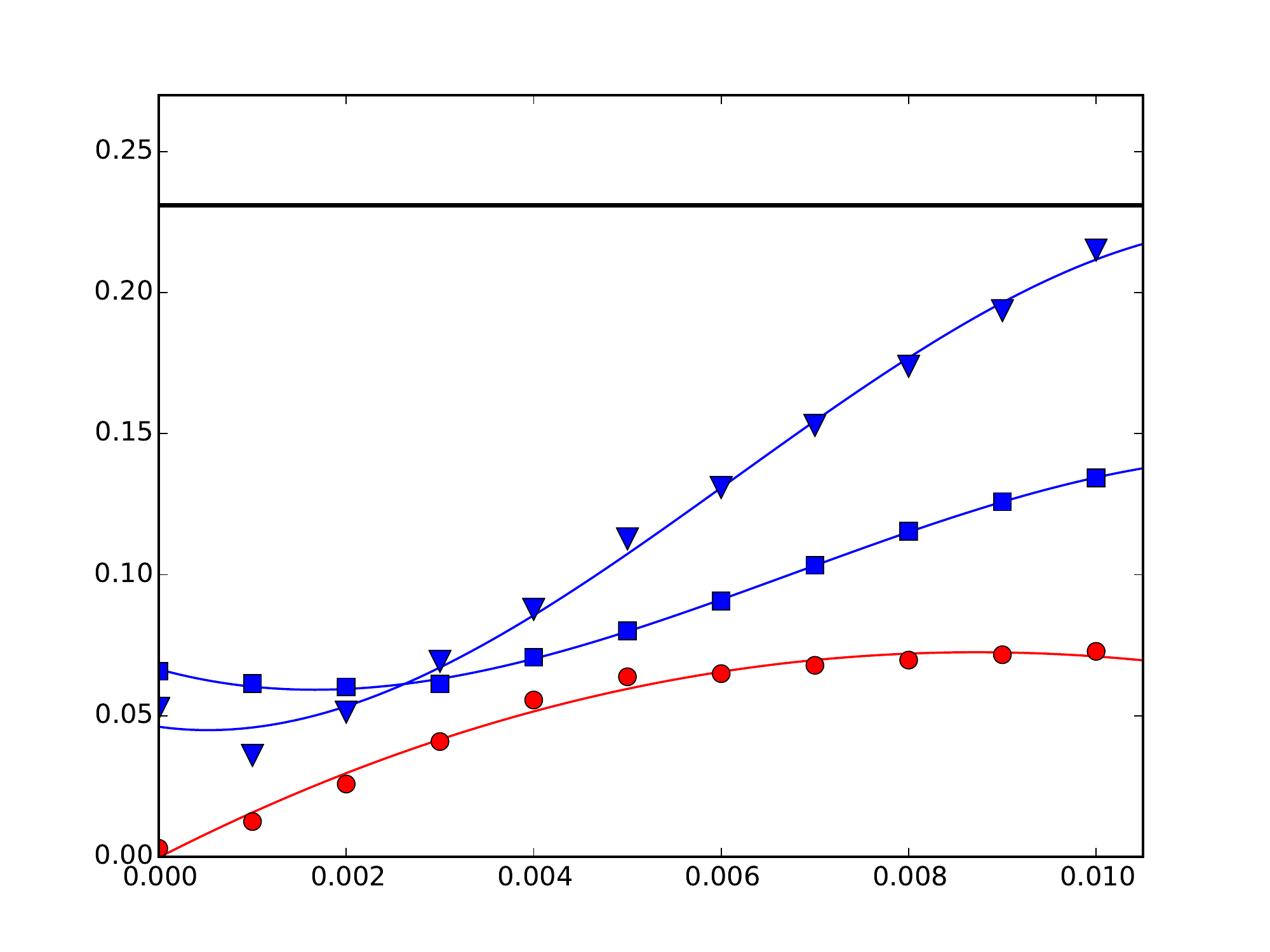}
\put(-160,0){$\sigma$}
\put(-330,100){{\small $\cuL(\bH_0,\hbH)^{1/2}$}}
    \caption{Risk $\cuL(\bH_0,\hbH)^{1/2}$ vs $\sigma$ for different reconstruction methods. Triangles (blue): anchor words  algorithm from \cite{arora2013practical}.
Squares (blue): minimizing the objective function (\ref{eq:StandardNMF}) using the projected gradient algorithm of \cite{lin2007projected}. Circles (red): 
archetypal reconstruction approach in this paper. Interpolating lines are just guides for the eye. The thick horizontal line corresponds to the trivial estimator $\hbH=0$.}\label{fig:Dvssigma}
\end{figure}
\subsection{Numerical experiments}

We implemented both the PALM algorithm described in the previous
section, and the two algorithms described in Appendix \ref{app:algo}. The outcomes are generally 
similar. 

Figures \ref{fig:projgradsignals_lownoise} and \ref{fig:oursignals_highnoise} repeat the experiment already described in the introduction.
We generate $n=250$ convex combinations of $r=4$ spectra $\bh_{0,1},\dots,\bh_{0,4}\in\reals^d$, $d=87$, this time adding white Gaussian noise with 
variance $\sigma^2$. We minimize the Lagrangian $\cuR_{\lambda}(\bH)$,
with\footnote{These values were chosen as to  approximately minimize the estimation error.}  $\lambda = 4$ (for Figure \ref{fig:projgradsignals_lownoise}) and $\lambda = 0.8$ 
(for Figure \ref{fig:oursignals_highnoise}). The  reconstructed spectra of the pure analytes
appear to be accurate and robust to noise.

In Figure \ref{fig:Dvssigma} we repeated the same experiment systematically for $10$ noise realizations for each noise level $\sigma$, and report
the resulting average loss. 
Among various reconstruction methods, the approach described in this paper seem to have good robustness to noise and achieves 
exact reconstruction as $\sigma\to 0$.

\section{Discussion}
\label{sec:Discussion}

\begin{figure}
 \includegraphics[width=0.475\textwidth]{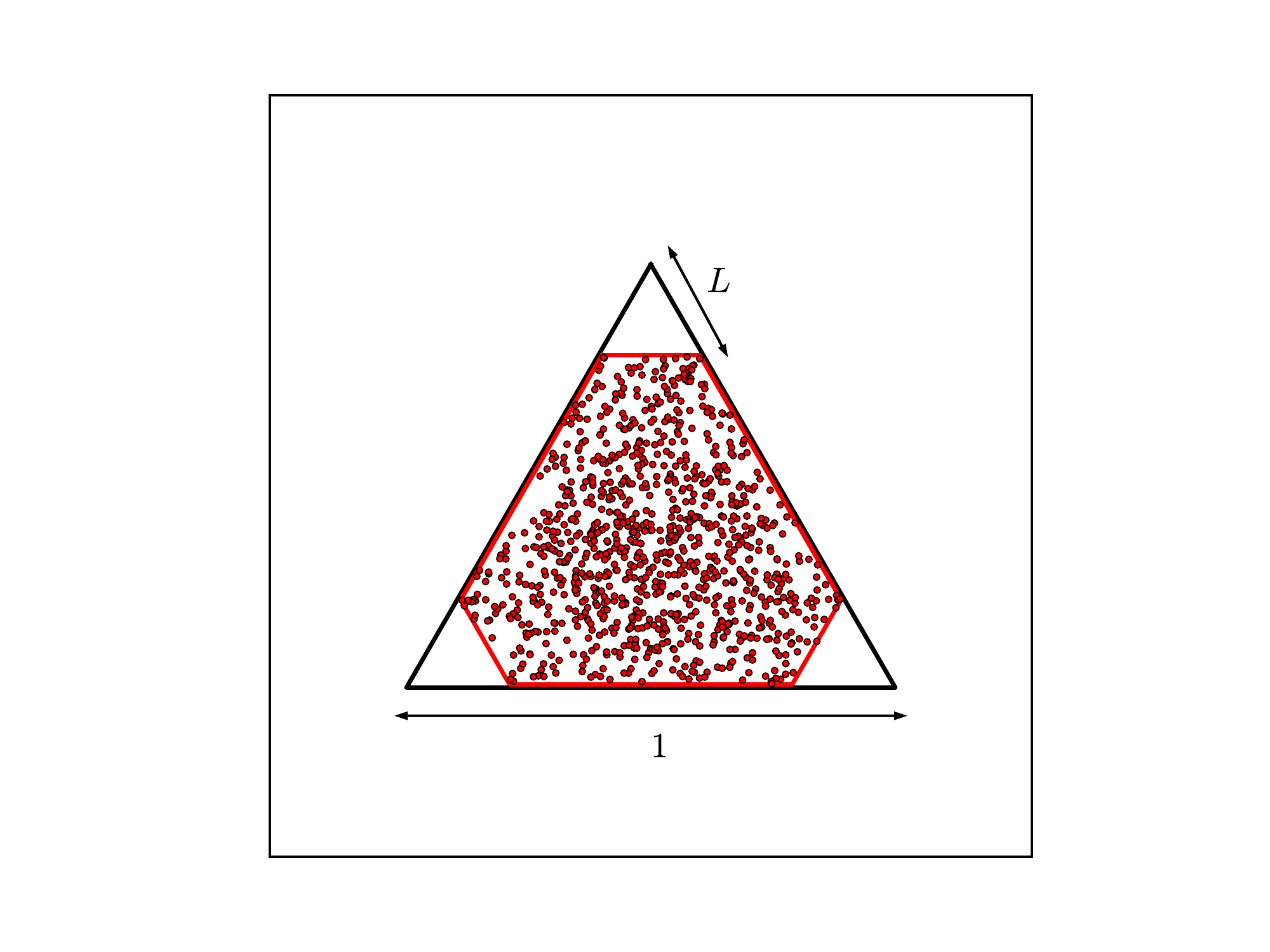}
 \includegraphics[width=0.475\textwidth]{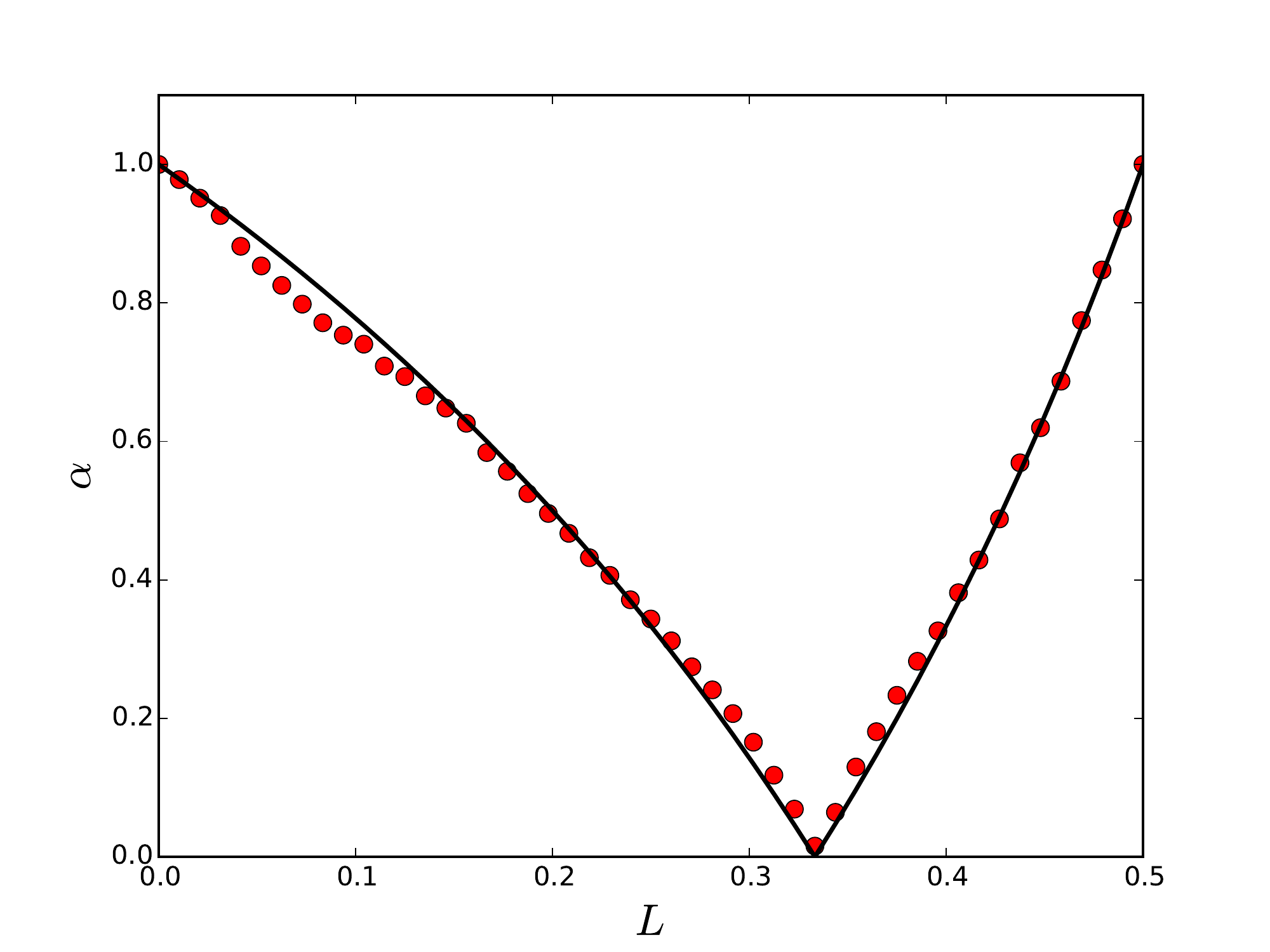}
    \caption{Numerical computation of the uniqueness parameter $\alpha$. Left: data geometry. The red hexagon corresponds to  $\conv(\bX)$, and
the black (equilateral) triangle to the archetypes  $\bH_0$, for $L<1/3$. For $L=1/3$ the archetypes are not unique, and for $L\in (1/3,1/2]$, they are given by an
equilateral triangle rotated by $\pi/3$ (pointing down). Right: numerical evaluation of the uniqueness constant (red circles). The continuous line corresponds to an analytical 
upper bound (triangle rotated by $\pi/3$ with respect to $\bH_0$).}\label{fig:Alphanum}
\end{figure}

We introduced a new optimization formulation of the non-negative matrix factorization problem.
In its Lagrangian formulation, our approach consists in minimizing the cost function $\cuR_{\lambda}(\bH)$ defined in Eq.~(\ref{eq:RLagrangian}). 
This encompasses applications in which only one of the factors is required to be non-negative.
A special case of this formulation ($\lambda\to\infty$) corresponds to the `archetypal analysis' of \cite{cutler1994archetypal}. In this case,
the archetype estimates coincide with a subset of the data points, which is appropriate only under the separability assumption of \cite{donoho2003does}.

Our main technical result (Theorem \ref{thm:Robust}) is a robustness guarantee for the reconstructed archetypes. This holds under 
the uniqueness assumption, which appears to hold for generic geometries of the dataset. In particular, while separability implies uniqueness (with optimal constant 
$\alpha=1$), uniqueness holds for non-separable data as well. To the best of our knowledge, similar robustness results have been obtained 
under separability \cite{recht2012factoring,arora2013practical,gillis2014fast,gillis2015semidefinite} (albeit these works obtain a better dependence on $r$). 
The only exception is the recent work of Ge and Zou \cite{ge2015intersecting}
who prove robustness under a `subset separability' condition, which provides a significant relaxation of  separability.
Under this condition, \cite{ge2015intersecting} develops a polynomial-time algorithm to estimate the archetypes by identifying and intersecting the
faces of $\conv(\bH_0)$. However, the algorithm  of  \cite{ge2015intersecting}  exploits collinearities to identify the faces, and this requires additional `genericity' assumptions.

Admittedly, the uniqueness constant $\alpha$ is difficult to evaluate analytically, even for simple geometries
of the data. However, by definition it does not vanish except in the case of multiple minimizers, and we expect it typically to be of
order one. Figure \ref{fig:Alphanum} illustrate this point by computing numerically $\alpha$ for a simple one-parameter family 
of geometries with $r=3$, $d=2$. The parameter $\alpha$ vanishes at a single point, corresponding to a degenerate problem with multiple solutions.

Finally,  several earlier works addressed the non-uniqueness problem in classical non-negative
matrix factorization. Among others, Miao and Qi \cite{miao2007endmember} penalize a matrix of archetypes $\bH$ by the corresponding volume. Closely
related to our work is the approach of M{\o}rup and Hansen \cite{morup2012archetypal} tha also builds on archetypal analysis. To the best of our knowledge, none of these
works establishes robustness of the proposed methods.

We conclude by mentioning three important problems that are not addressed by this paper:
$(1)$ Are there natural condition under which  the risk function  $\cuR_{\lambda}(\bH)$ of Eq.~(\ref{eq:RLagrangian}) can be optimized
in polynomial time? We only provided an algorithm that is guaranteed to converge to a critical point.
$(2)$ We assumed the rank $r$ to be known. In practice it will need to be estimated from the data.
$(3)$ Similarly, the regularization parameter $\lambda$ should be chosen from  data. 

\section*{Acknowledgements}
This work was partially supported by the NSF grant CCF-1319979 and a Stanford Graduate Fellowship.

\bibliographystyle{amsalpha}

\newcommand{\etalchar}[1]{$^{#1}$}
\providecommand{\bysame}{\leavevmode\hbox to3em{\hrulefill}\thinspace}
\providecommand{\MR}{\relax\ifhmode\unskip\space\fi MR }
\providecommand{\MRhref}[2]{%
  \href{http://www.ams.org/mathscinet-getitem?mr=#1}{#2}
}
\providecommand{\href}[2]{#2}

\addcontentsline{toc}{section}{References}

\newpage

\appendix

\section{Further details on numerical experiments}
\label{app:Numerical}

The data in Figures \ref{fig:spectra-no-noise}, \ref{fig:projgradsignals_lownoise}, and \ref{fig:oursignals_highnoise} were generated as follows. We retrieved 
infrared reflection spectra of caffeine, sucrose, lactose and trioctanoin from the NIST Chemistry WebBook dataset \cite{nist}. We restricted these spectra to the  
wavenumbers between  $1186\;  {\rm cm}^{-1}$ and $1530\; {\rm cm}^{-1}$, and denote by $\bh_{0,1},\dots,\bh_{0,4}\in \reals^d$, $d=87$ the vector representations of these spectra.
We then generates data $\bx_i\in\reals^d$, $i\le n = 250$ by letting  
\begin{align}
\bx_i = \sum_{\ell=1}^4 w_{i,\ell}\bh_{\ell}+\bz_i\, ,
\end{align}
where $\bz_i\sim \normal(0,\sigma^2\id_{d})$ are i.i.d. Gaussian noise vectors. The weights $\bw_{i} =(w_{i,\ell})_{\ell\le 4}$ were generated as follows. The weight vectors $\{\bw_i\}_{1\le i \le 9}$ are generated such that they have $2$ nonzero entries. In other words, $9$ data points
are on one dimensional facets of the polytope generated by $\bh_{0,1},\dots,\bh_{0,4}$. In order to randomly generate these weight vectors, 
for each $1\le i \le 9$, a pair of indices $(\ell_1, \ell_2)$ between $1$ and $4$ is chosen uniformly at random. Then 
$\{\widetilde\bw\}_{1\le i \le 9}$, $\widetilde\bw\in \reals^2$ are generated as 
independent Dirichlet random vectors with parameter $(5,5)$. Then we let $w_{i,\ell_1} = \tilde w_{i,1}$ and $w_{i,\ell_2} = \tilde w_{i,2}$ for
$1\le i \le 9$. The weight vectors $\{\bw_i\}_{10\le i \le 20}$ each have $3$ nonzero entries. Similar to above, for each of these weight 
vectors a $3$-tuple of indices $(\ell_1, \ell_2, \ell_3)$ between $1$ and $4$ is chosen uniformly at random. Then we 
let $w_{i,\ell_1} = \tilde w_{i,1}$, $w_{i,\ell_2} = \tilde w_{i,2}$, $w_{i,\ell_3} = \tilde w_{i,4}$ for $10\le i \le 20$, where 
$\{\widetilde\bw\}_{10\le i \le 20}$, $\widetilde\bw\in \reals^3$ are i.i.d. Dirichlet random vectors with parameter $(5,5,5)$.
The rest of the weight vectors have cardinality equal to $4$. Hence, for $21 \leq i\le 250$, $\bw_i$ are generated
as i.i.d. Dirichlet random vectors with parameter $(5,5,5,5)$.

\section{Proof of Theorem \ref{thm:Robust}}
\label{app:Proof}

In this appendix we prove Theorem \ref{thm:Robust}. We start by recalling some notations already defined in the main text, 
and introducing some new ones. We will then state a stronger form of the theorem (with better dependence on the problem geometry in some regimes).
Finally, we will present the actual proof.

Throughout this appendix, we assume the square loss $D(\bx,\by) = \|\bx-\by\|_2^2$.

\subsection{Notations and definitions}

We use bold capital letters (e.g. $\bA$, $\bB$, $\bC$,\dots) for matrices, bold lower case  for vectors (e.g. $\bx$, $\by$, \dots)
and plain lower case for scalars ($aa$, $b$, $c$ and so on).
In particular, $\be_i \in \reals^d$ denotes the $i$'th vector in the canonical basis, $E^{r,d} = \{\be_1,\be_2,\dots,\be_r\}$ and 
for $r\leq d$, $\bE_{r,d}\in \{0,1\}^{r\times d}$ is the matrix whose $i$'th column is $\be_i$, and whose columns after the $r$-th one are equal to $\bzero$.
For a matrix $\bX$, $\bX_{i,.}$ and $\bX_{.,i}$ are its $i$'th row and column, respectively. 

As in the main text, we denote by $\Delta^{m}$ the $m$-dimensional standard simplex, i.e. $\Delta^m = \{\bx\in \reals_{\geq0}^{m}, \langle\bx,\one\rangle = 1\}$, 
where $\one \in \reals^m$ is the all ones vector. 
For a matrix $\bH \in \reals^{r\times d}$, we use
$\sigma_{\max}(\bH)$, $\sigma_{\min}(\bH)$ to denote its largest and smallest nonzero singular values
and $\kappa(\bH)= \sigma_{\max}(\bH)/\sigma_{\min}(\bH)$ to denote its condition number.
We denote by $\conv(\bH), \aff(\bH)$ the convex hull and 
the affine hull of the rows of $\bH$, respectively. In other words,
\begin{align}
&\conv(\bH) = \{\bx\in \reals^{d}:\bx = \bH^{\sT}\bpi, \bpi \in \Delta^{r} \},\\
&\aff(\bH) = \{\bx\in \reals^{d}:\bx = \bH^{\sT}\balpha, \langle \one, \balpha \rangle = 1 \}.
\end{align}
We denote by $Q_{r,n}$ is the set of $r$ by $n$ row stochastic matrices. Namely,
\begin{align}
Q_{r,n} = \left\{\bPi\in \reals_{\geq 0}^{r\times n}: \langle\bPi_{i,.},\one\rangle = 1\right\}.
\end{align}
with use $Q_r\equiv Q_{r,r}$. Further, $S_r$ is defined as
\begin{align}
S_r = \left\{\bPi\in Q_r: \Pi_{i,j}\in \{0,1\}\right\}.
\end{align}

As a consequence, given $\bX\in \reals^{n\times d}$, $\bH_1,\bH_2\in \reals^{r\times d}$, the loss functions $\cuD(\,\cdot\,,\,\cdot\,)$ and $\cuL(\,\cdot\,,\,\cdot\,)$
take the form
\begin{align}
\cuD(\bH_1,\bX) &= \min_{\bPi\in Q_{r,n}} \|\bH_1 - \bPi \bX\|^2_F,\\
\cuL(\bH_1,\bH_2) &= \min_{\bPi\in S_r} \|\bH_1 - \bPi \bH_2\|_F^2.
\end{align}

We use $\Ball_m(\rho)$ to denote the closed ball with radius $\rho$ in $m$ dimensions, centered at $0$.
 In addition, for $\bH \in \reals^{m\times d}$ we define the $\rho$-neighborhood of $\conv(\bH)$ as
\begin{align}
\Ball_r(\rho;\bH) := \{\bx \in \reals^d: \cuD(\bx,\bH)\leq\rho^2\}.
\end{align}
For a convex set $\mathcal C$ we denote the set of its extremal points by $\ext(\mathcal C)$ and the projection of a point $\bx \in \reals^{d}$ onto $\mathcal C$ by 
$\bPi_\mathcal C(\bx)$. Namely,
\begin{align}
\bPi_{\mathcal C}(\bx) = \arg\min_{\by \in \mathcal C}\|\bx - \by\|_2.
\end{align}
Also, for a matrix $\bX \in \reals^{n\times d}$, and a mapping (not necessarily linear)
$\bP:\reals^d \rightarrow \reals^d$, $\bP(\bX) \in \reals^{n\times d}$
is the matrix whose $i$'th row is $\bP(\bX_{i,.})$.

\subsection{Theorem statement}

The statement below provides  more detailed result with respect to the one in Theorem \ref{thm:Robust}.

\begin{theorem}\label{thm:Robust2}
Assume $\bX = \bW_0\bH_0+ \bZ$ where the factorization $\bX_0=\bW_0\bH_0$ satisfies the uniqueness assumption with parameter
$\alpha>0$, and that $\conv(\bX_0)$ has internal radius $\mu>0$.
Consider the estimator $\hbH$ defined by Eq.~(\ref{eq:HardNoise}), with $D(\bx,\by) = \|\bx-\by\|_2^2$ (square loss) and $\delta = \max_{i\le n} \|\bZ_{i,\cdot}\|_2$.
 If 
\begin{align}
\max_{i\le n} \|\bZ_{i,\cdot}\|_2\le \frac{\alpha\mu}{30 r^{3/2}}\, ,
\end{align}
then, setting $\delta = \max_{i\le n} \|\bZ_{i,\cdot}\|_2$ in the problem \eqref{eq:HardNoise} we get
\begin{align}
\cuL(\bH_0,\hbH)\le \frac{C^2_*\,  r^{5}}{\alpha^2} \max_{i\le n} \|\bZ_{i,\cdot}\|^2_2\, ,
\end{align}
where $C_*$ is a coefficient that depends uniquely on the geometry of $\bH_0$, $\bX_0$, namely $C_* = 120(\sigma_{\rm max}(\bH_0)/\mu)\cdot \max(1, \kappa(\bH_0)/\sqrt{r})$.

 Further, if
\begin{align}
\max_{i\le n} \|\bZ_{i,\cdot}\|_2\le \frac{\alpha\mu}{330\kappa(\bH_0)r^{5/2}},
\end{align}
then, setting $\delta = \max_{i\le n} \|\bZ_{i,\cdot}\|_2$ in the problem \eqref{eq:HardNoise} we get
\begin{align}
\cuL(\bH_0,\hbH)\le \frac{C_{**}^{2}\,r^4}{\alpha^2} \max_{i\le n} \|\bZ_{i,\cdot}\|^2_2\, ,
\end{align}
where $C_{**}  = 120 \max(\kappa(\bH_0), (\sigma_{\max}(\bH_0)/r+\|\bz_0\|_2)/(\mu r^{1/2}))\cdot\max (1, \kappa(\bH_0)/\sqrt{r})$.

\end{theorem}

\subsection{Proof}

\subsubsection{Lemmas}
 \begin{lemma}
\label{lemma:cone}
Let $\mathcal R$ be a convex set and $\mathcal C$ be a convex cone. Define
\begin{align}
\gamma_\mathcal C = \max_{\|\bu\|_2=1}\min_{\bv\in \mathcal C, \|\bv\|_2=1}\langle \bu,\bv\rangle.
\end{align}
We have
\begin{align}
\min_{\bx\in \mathcal R}\|\bx\|_2 + (1+\gamma_\mathcal C)\max_{x\in \ext(\mathcal R)}\|\bx-\bPi_\mathcal C(\bx)\|_2 \geq \gamma_\mathcal C \min_{\bx \in \ext(\mathcal R)}\|\bx\|_2.
\end{align}
\end{lemma}
An illustration of this lemma in the case of $\mathcal R \subset \mathcal C$
is given  in Figure \ref{fig:lemma_cone}. Note that, $\gamma_{\mathcal C}$
measures the pointedness of the cone $\mathcal C$. Geometrically (for ${\mathcal R}\subseteq {\mathcal C}$) the lemma states
that the cosine of the 
angle between $\arg\min_{\bx\in \mathcal R}\|\bx\|_2$ and $\arg\min_{\bx\in \mathcal \ext(R)}\|\bx\|_2$
is smaller than $\gamma_\mathcal C$.

\begin{figure}[!h]
    \centering
    \includegraphics[width=0.7\textwidth]{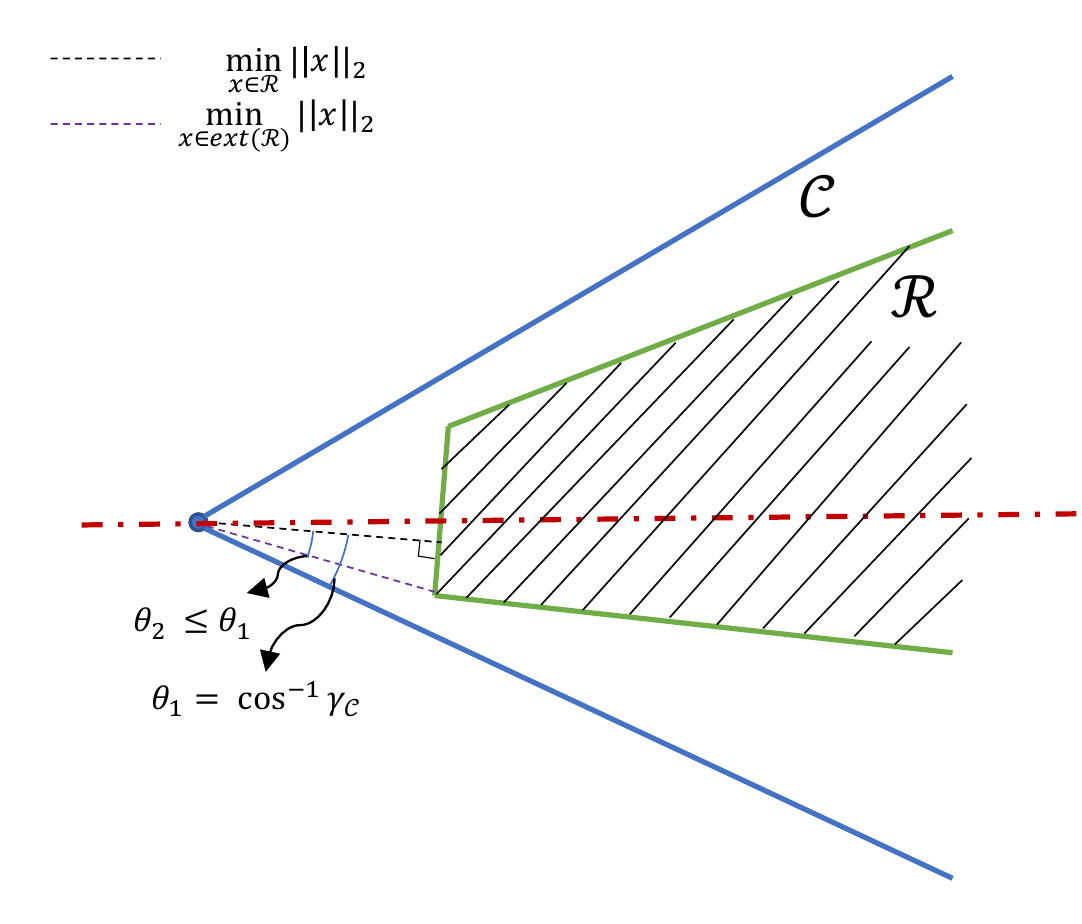}
    \hspace{-15pt}
     \vspace{-15pt}
    \caption{Picture of Lemma \ref{lemma:cone}, in the case, $\mathcal R \subset \mathcal C$.}\label{fig:lemma_cone}
\end{figure}

\begin{proof}
We write
\begin{align}
\min_{\bx \in \mathcal R} \|\bx\|_2 = \min_{\bx \in \mathcal R}\max_{\|\bu\|_2 = 1}\left\langle\bu,\bx\right\rangle \geq \max_{\|\bu\|_2=1}\min_{\bx\in \mathcal R}\langle \bu,\bx\rangle = \max_{\|\bu\|_2 = 1}\min_{\bx \in \ext(\mathcal R)}\langle \bu, \bx \rangle.
\end{align}
Replacing 
\begin{align}
\bx = \bPi_{\mathcal C}(\bx) + \left(\bx - \bPi_{\mathcal C}(\bx)\right),
\end{align}
we get
\begin{align}
\min_{\bx \in \mathcal R}\|\bx\|_2 &\geq \max_{\|\bu\|_2 = 1}\min_{\bx\in \ext (\mathcal R)}\left\langle\bu,\bPi_\mathcal C(\bx)+(\bx-\bPi_{\mathcal C}(\bx))\right\rangle \\
&\geq\max_{\|\bu\|_2 = 1}\min_{\bx \in \ext(\mathcal R)}\left\langle\bu,\bPi_{\mathcal C}(\bx)\right\rangle - \max_{\bx \in \ext(\mathcal R)}\|\bx - \bPi_{\mathcal C}(\bx)\|_2.
\end{align}
Hence, using the definition of $\gamma_\mathcal C$, we have 
\begin{align}
\min_{\bx \in \mathcal R} \|\bx\|_2 \geq \gamma_\mathcal C \min_{\bx\in \ext(\mathcal R)}\|\bPi_\mathcal C(\bx)\|_2-\max_{\bx \in \ext(\mathcal R)}\|\bx - \bPi_\mathcal C(\bx)\|_2.                                                                                        
\end{align}
Note that 
\begin{align}
\|\bPi_\mathcal C(\bx)\|_2 \geq \|\bx\|_2 - \|\bx-\bPi_\mathcal C(\bx)\|_2.
\end{align}
Therefore,
\begin{align}
\min_{\bx \in \mathcal R} \|\bx\|_2 \geq \gamma_{\mathcal C}\min_{\bx \in \ext(\mathcal R)}\|\bx\|_2 - (1+\gamma_\mathcal C)\max_{\bx \in \ext(\mathcal R)}\|\bx-\bPi_\mathcal C(\bx)\|_2,
\end{align}
and this completes the proof.
\end{proof}

The next lemma is a consequence of Lemma \ref{lemma:cone}.
\begin{lemma}
\label{lemma:dandc}
Let $\bH, \bH_0\in \reals^{r\times d}$, $r\leq d$,
be matrices with linearly independent rows. We have
\begin{align}
\cuL(\bH_0,\bH)^{1/2}\leq \sqrt{2}\kappa(\bH_0)\cuD(\bH_0,\bH)^{1/2} + (1+\sqrt{2})\sqrt{r}\cuD(\bH,\bH_0)^{1/2}\, . 
\end{align}
\end{lemma}

\begin{proof}
Consider the cone $\mathcal C_1\subset \reals^d$, generated by vectors $\be_2-\be_1, \dots, \be_r-\be_1\in \reals^d$, i.e,
\begin{align}
\mathcal C_1 = \left\{\bv \in \reals^d; \bv = \sum_{i=2}^r{v_i(\be_i - \be_1)}, v_i \geq 0\right\}.
\end{align} 
For $\bv \in \mathcal C_1, \|\bv\|_2 = 1$ we have
\begin{align}
\bv = \left(-\langle\one,\bx\rangle,\bx,0,0, \dots, 0\right),
\end{align}
where $\bx \in \reals_{\geq 0}^{r-1}$ and
\begin{align}
\|\bx\|_2^2 + \left\langle\one,\bx\right\rangle^2 = 1.
\end{align}
Since, $\langle\one,\bx\rangle = \|\bx\|_1 \geq \|\bx\|_2$, we get $\langle \one,\bx\rangle \geq 1/\sqrt{2} $. Thus, for $\bu = -\be_1$, we have $\langle\bu,\bv\rangle \geq 1/\sqrt{2}$. 
Therefore, for $\gamma_{{\mathcal C}_1}$ defined as in Lemma \ref{lemma:cone}, we have $\gamma_{\mathcal C_1} \geq 1/\sqrt{2}$. In addition, by symmetry,
for $i \in \{1,2,\dots,r\}$, for the cone $\mathcal C_i\subset \reals^d$, generated by vectors $\be_1-\be_i, \be_2-\be_i, \dots, \be_r-\be_i\in \reals^d$ we have $\gamma_{\mathcal C_i} = \gamma \geq 1/\sqrt{2}$.
Hence, using Lemma \ref{lemma:cone} for $\bH \in \reals^{r\times d}$, $\mathcal R = \conv(\bH)-\be_j$ (the set obtained by translating $\conv(\bH)$ by $-\be_j$), $\mathcal C = \mathcal C_j$ we get for $j = 1,2,\dots, r$
\begin{align}
\min_{\bq\in \Delta^r}\|\be_j - {\bH}^\sT\bq\|_2 &\geq \gamma\min_{\bq\in E^{r,r}}\|\be_j - {\bH}^\sT\bq\|_2 - (1+\gamma)\max_{i \in [r]}\min_{\bq\in \reals_{\geq 0}^r}\|\bH_{i,.}^\sT - \be_j - \bE_{r,d}^\sT\bq + \be_j \langle \one,\bq\rangle\|_2\\
&\geq \gamma\min_{\bq\in E^{r,r}}\|\be_j - {\bH}^\sT\bq\|_2 - (1+\gamma)\max_{i\in [r]}\min_{\bq\in \Delta^r}\|\bH_{i,.}^\sT-\bE_{r,d}^\sT\bq\|_2.
\end{align}
Hence, 
\begin{align}
\sum_{j=1}^r \min_{\bq\in \Delta^r}\|\be_j - {\bH}^\sT\bq\|_2^2 &\geq \gamma^2 \sum_{j=1}^r \min_{\bq\in E^{r,r}}\|\be_j - {\bH}^\sT\bq\|_2^2 + (1+\gamma)^2r\max_{i \in [r]}\min_{\bq\in \Delta^r}\|\bH_{i,.}^\sT-\bE_{r,d}^\sT\bq\|_2^2 \nonumber \\ 
&- 2\gamma(1+\gamma) \left(\max_{i\in [r]}\min_{\bq\in \Delta^r}\|\bH_{i,.}^\sT-\bE_{r,d}^\sT\bq\|_2\right)\sum_{j=1}^r \min_{\bq\in E^{r,r}}\|\be_j - {\bH}^\sT\bq\|_2\\
&\geq \Bigg[\gamma\bigg(\sum_{j=1}^r \min_{\bq\in E^{r,r}}\|\be_j - {\bH}^\sT\bq\|_2^2 \bigg)^{1/2} \nonumber\\
&- (1+\gamma)\sqrt{r}\bigg(\max_{i \in [r]}\min_{\bq\in \Delta^r}\|\bH_{i,.}^\sT-\bE_{r,d}^\sT\bq\|_2\bigg)\Bigg]^2.
\end{align}
Therefore,
\begin{align}
\min_{\bQ\in Q_r}\|\bE_{r,d} - \bQ\bH\|_F \geq \gamma \min_{\bQ\in S_r}\|\bE_{r,d} - \bQ\bH\|_F - (1+\gamma)\sqrt{r}\max_{i \in [r]}\min_{\bq\in \Delta^r}\|\bH_{i,.}^\sT-\bE_{r,d}^\sT\bq\|_2.
\end{align}
Now consider $\bH_0 \in \reals^{r\times d}$ where $\bH_0 = \bE_{r,d}\bM$, $\bH = \bY\bM$, where $\bM\in \reals^{d\times d}$ is invertible. We have
\begin{align}
\cuD(\bH_0,\bH)^{1/2} &= \min_{\bQ\in Q_r}\|\bH_0 - \bQ\bH\|_F = \min_{\bQ\in Q_r}\|(\bE_{r,d} - \bQ\bY)\bM\|_F\\
& \geq \sigma_{\min}(\bM)\min_{\bQ\in Q_r}\|\bE_{r,d} - \bQ\bY\|_F \\
& \geq \gamma\sigma_{\min}(\bM)\min_{\bQ\in S_r}\|\bE_{r,d} - \bQ\bY\|_F - \sigma_{\min}(\bM)\sqrt{r}(1+\gamma)\max_{i \in [r]}\min_{\bq\in \Delta^r}\|{\bY_{i,.}^\sT}-\bE_{r,d}^\sT\bq\|_2\\
&= \gamma\sigma_{\min}(\bM)\min_{\bQ\in S_r}\|(\bH_0-\bQ\bH)\bM^{-1}\|_F \nonumber\\
&- \sigma_{\min}(\bM)\sqrt{r}(1+\gamma)\max_{i \in [r]}\min_{\bq\in \Delta^r}\|{(\bM^{-1}})^\sT(\bH_{i,.}^\sT-\bH_0^\sT\bq)\|_2.
\end{align}
Thus, using the fact that $\sigma_{\max}(\bM)/\sigma_{\min}(\bM) = \kappa(\bM) = \kappa(\bH_0)$,
\begin{align}
\cuD(\bH_0,\bH)^{1/2} &\geq \frac{\gamma}{\kappa(\bH_0)}\cuL(\bH_0,\bH)^{1/2} - \frac{(1+\gamma)\sqrt{r}}{\kappa(\bH_0)}\max_{i \in [r]}\min_{\bq\in \Delta^r}\|\bH_{i,.}^\sT-\bH_0^\sT\bq\|_2\\
&\geq \frac{\gamma}{\kappa(\bH_0)}\cuL(\bH_0,\bH)^{1/2} - \frac{(1+\gamma)\sqrt{r}}{\kappa(\bH_0)}\cuD(\bH,\bH_0)^{1/2}\, .
\end{align}
Therefore,
\begin{align}
\cuL(\bH_0,\bH)^{1/2} \leq \frac{\kappa(\bH_0)}{\gamma}\cuD(\bH_0,\bH)^{1/2} + \frac{(1+\gamma)\sqrt{r}}{\gamma}\cuD(\bH,\bH_0)^{1/2}\, .
\end{align}
Finally, note that the function $f(x) = (1+x)/x$ is monotone decreasing over $\reals_{> 0}$.
Hence, for $\gamma \geq 1/\sqrt{2}$, $(1+\gamma)/\gamma \leq 1+\sqrt{2}$.
Therefore, we get
\begin{align}
\cuL(\bH_0,\bH)^{1/2}\leq \sqrt{2}\kappa(\bH_0)\cuD(\bH_0,\bH)^{1/2} + (1+\sqrt{2})\sqrt{r}\cuD(\bH,\bH_0)^{1/2}
\end{align}
and this completes the proof.
\end{proof}

We continue with the following lemmas on the condition number of the matrix $\bH$.
\begin{lemma}
\label{lemma:condition}
Let $\bH_0,\bH\in \reals^{r\times d}$, $r\leq d$,with $\bH$ having full row rank. 
We have
\begin{align}
\sigma_{\max}(\bH) \leq \cuD(\bH,\bH_0)^{1/2}+\sqrt{r}\sigma_{\max}(\bH_0),\label{eq:LemmaCondition1}.
\end{align}
In addition, if
\begin{align}
\cuD(\bH_0,\bH)^{1/2} \leq \frac{\sigma_{\min}(\bH_0)}{2},
\end{align}
then
\begin{align}
&\kappa(\bH) \leq \frac{2r\sigma_{\max}(\bH_0)+2\cuD(\bH,\bH_0)^{1/2}\sqrt{r}}{\sigma_{\min}(\bH_0)}. \label{eq:LemmaCondition2}
\end{align}
Further, if 
\begin{align}
\cuD(\bH,\bH_0)^{1/2}+\cuD(\bH_0,\bH)^{1/2} \leq \frac{\sigma_{\min}(\bH_0)}{6\sqrt{r}},
\end{align}
then 
\begin{align}
&\sigma_{\max}(\bH)\leq 2\sigma_{\max}(\bH_0),\\
&\kappa(\bH) \leq (7/2)\kappa(\bH_0).
\end{align}
\end{lemma}
\begin{proof}
For the sake of simplicity, we will write $\cuD_1 = \cuD(\bH,\bH_0)^{1/2}$, $\cuD_2 = \cuD(\bH_0,\bH)^{1/2}$
Note that using the assumptions of Lemma \ref{lemma:condition} we have
\begin{align}
\label{eq:w,w0}
\begin{split}
&\bH_0 = \bP\bH + \bA_2; \quad \|\bA_2\|_F= \cuD_2,\\
&\bH = \bR\bH_0 + \bA_1; \quad \|\bA_1\|_F= \cuD_1,
\end{split}
\end{align}
where $\bP, \bR \in \reals_{\geq 0}^{r\times r}$ are row-stochastic matrices and  $\bA_1, \bA_2 \in \reals^{r\times d}$. Also, $\sigma_{\max}(\bA_1)\leq \|\bA_1\|_F= \cuD_1$,
$\sigma_{\max}(\bA_2)\leq \|\bA_2\|_F= \cuD_2$. Therefore,
\begin{align}
\sigma_{\max}(\bP)\sigma_{\min}(\bH) \geq \sigma_{\min}(\bP\bH) \geq \sigma_{\min}(\bH_0)-\sigma_{\max}(\bA_2) \geq \sigma_{\min}(\bH_0)-\cuD_2.
\end{align}
In addition, note that for a row stochastic matrix $\bP \in Q_r$, we have
\begin{align}
\sigma_{\max}(\bP)\leq \|\bP\|_F = \left(\sum_{i=1}^r\|\bP_{i,.}\|_2^2\right)^{1/2} \leq \left(\sum_{i=1}^r\|\bP_{i,.}\|_1^2\right)^{1/2} \leq \sqrt{r}.
\end{align}
Hence, for $\cuD_2\leq \sigma_{\min}(\bH_0)$ we get
\begin{align}
\label{eq:sigma_minH}
\sigma_{\min}(\bH)\geq \frac{\sigma_{\min}(\bH_0) - \cuD_2}{\sqrt{r}}\, .
\end{align}
In addition,
\begin{align}
\label{eq:sigma_maxH}
\sigma_{\max}(\bH) \leq \sigma_{\max}(\bR\bH_0) + \sigma_{\max}(\bA_1) \leq \sigma_{\max}(\bR)\sigma_{\max}(\bH_0) + \cuD_1 \leq \sqrt{r}\sigma_{\max}(\bH_0)+\cuD_1.
\end{align}
Hence, using \eqref{eq:sigma_minH}, \eqref{eq:sigma_maxH}, for $\cuD_2\leq \sigma_{\min}(\bH_0)$ we have
\begin{align}
\kappa(\bH) \leq \frac{r\sigma_{\max}(\bH_0)+\cuD_1\sqrt{r}}{\sigma_{\min}(\bH_0)-\cuD_2}.
\end{align}
Thus, for $\cuD_2\leq \sigma_{\min}(\bH_0)/2$, we get Eqs.~(\ref{eq:LemmaCondition1}), (\ref{eq:LemmaCondition2}). 

Now assume that $\cuD_1+\cuD_2 \leq \sigma_{\min}(\bH_0)/(6\sqrt{r})$. In this case, using \eqref{eq:w,w0} we have 
\begin{align}
\bH_0 = \bP(\bR\bH_0 + \bA_1) + \bA_2.
\end{align}
Therefore, 
\begin{align}
(\Id - \bP\bR)\bH_0 = \bP\bA_1 + \bA_2,
\end{align}
hence,
\begin{align}
\Id - \bP\bR = (\bP\bA_1 + \bA_2)\bH_0^{\dagger}
\end{align}
and
\begin{align}
\bP\bR = \Id - \bP\bA_1\bH_0^{\dagger} - \bA_2\bH_0^{\dagger}.
\end{align}
where $\bH_0^{\dagger}$ is the right inverse of matrix $\bH_0$. Note that 
\begin{align}
\sigma_{\max}(\bH_0^{\dagger}) = \sigma_{\min}(\bH_0)^{-1}.
\end{align}
By permuting the rows and columns of $\bH_0$, without loss of generality,
we can assume that $R_{ii} = \ \|\bR_{.,i}\|_\infty$. We can write
\begin{align}
R_{ii} \geq \left\langle\bP_{i,.},\bR_{.,i}\right\rangle &= 1-(\bP\bA_1\bH_0^{\dagger})_{ii}-(\bA_2\bH_0^{\dagger})_{ii}\\
&\geq 1-\|(\bP\bA_1\bH_0^{\dagger})_{i,.}\|_2 - \|(\bA_2\bH_0^{\dagger})_{i,.}\|_2\\
&\geq 1-\max_{\bu\in \Delta^r}\|\bA_1^\sT \bu\|_2\sigma_{\max}(\bH_0^{\dagger}) - \|(\bA_2)_{i,.}\|_2\sigma_{\max}(\bH_0^{\dagger})\\
&\geq 1-\max_{\bu\in \Delta^r}\|\bu\|_2\sigma_{\max}(\bA_1)\sigma_{\max}(\bH_0^{\dagger}) - \|\bA_2\|_F\sigma_{\max}(\bH_0^{\dagger})\\
&\geq 1- \frac{\cuD_1+\cuD_2}{\sigma_{\min}(\bH_0)}. 
\label{eq:riibiggerthan}
\end{align}
Hence, for all $i,j\in [r], i\neq j$, since $\bR$ is row-stochastic, 
\begin{align}
R_{ji} \leq \frac{\cuD_1+\cuD_2}{\sigma_{\min}(\bH_0)}.
\end{align}
Thus,
\begin{align}
\left\langle\bP_{i,.},\bR_{.,i}\right\rangle = R_{ii}P_{ii} + \sum_{j\neq i}P_{ij}R_{ji} &\leq R_{ii}P_{ii} + \left(\max_{j\neq i}R_{ji}\right)\sum_{j\neq i}P_{ij}\\
&\leq P_{ii}+\frac{\cuD_1+\cuD_2}{\sigma_{\min}(\bH_0)}(1-P_{ii}).
\end{align}
Therefore, using \eqref{eq:riibiggerthan},
\begin{align}
P_{ii} \geq \frac{\sigma_{\min}(\bH_0)-2(\cuD_1+\cuD_2)}{\sigma_{\min}(\bH_0)-(\cuD_1+\cuD_2)}.
\end{align}
Thus, we can write
\begin{align}
\bP = \Id + \Delta; \quad \|\Delta_{i,.}\|_1 \leq \frac{2(\cuD_1+\cuD_2)}{\sigma_{\min}(\bH_0)-(\cuD_1 + \cuD_2)}.
\end{align}
Therefore,
\begin{align}
\sigma_{\max}(\Delta) \leq \|\Delta\|_F = \left(\sum_{i=1}^r\|\Delta_{i,.}\|_2^2\right)^{1/2} \leq \left(\sum_{i=1}^r\|\Delta_{i,.}\|_1^2\right)^{1/2} \leq \frac{2(\cuD_1+\cuD_2)\sqrt{r}}{\sigma_{\min}(\bH_0)-(\cuD_1 + \cuD_2)}.
\end{align}
Hence,
\begin{align}
&\sigma_{\max}(\bP) \leq 1+\frac{2\sqrt{r}(\cuD_1+\cuD_2)}{\sigma_{\min}(\bH_0)-(\cuD_1+\cuD_2)},
&\sigma_{\min}(\bP) \geq 1-\frac{2\sqrt{r}(\cuD_1+\cuD_2)}{\sigma_{\min}(\bH_0)-(\cuD_1+\cuD_2)}.
\end{align}
From \eqref{eq:w,w0} we have $\sigma_{\min}(\bP\bH) \geq \sigma_{\min}(\bH_0)-\cuD_2$. Using 
$\sigma_{\min}(\bP\bH)\leq \sigma_{\max}(\bP)\sigma_{\min}(\bH)$, we get
\begin{align}
\sigma_{\min}(\bH) \geq \frac{(\sigma_{\min}(\bH_0)-\cuD_2)(\sigma_{\min}(\bH_0)-(\cuD_1+\cuD_2))}{\sigma_{\min}(\bH_0)-(\cuD_1+\cuD_2)+2\sqrt{r}(\cuD_1+\cuD_2)}.
\end{align}
Further, from \eqref{eq:w,w0} we have $\sigma_{\max}(\bP\bH) \leq \sigma_{\max}(\bH_0)+\cuD_2$. Using $\sigma_{\max}(\bP\bH) \geq \sigma_{\min}(\bP)\sigma_{\max}(\bH)$, we get
\begin{align}
\sigma_{\max}(\bH) \leq \frac{(\sigma_{\max}(\bH_0)+\cuD_2)(\sigma_{\min}(\bH_0)-(\cuD_1+\cuD_2))}{\sigma_{\min}(\bH_0)-(\cuD_1+\cuD_2)-2\sqrt{r}(\cuD_1+\cuD_2)}.
\end{align}
Hence, for $\cuD_1+\cuD_2\leq \sigma_{\min}(\bH_0)/(6\sqrt{r})$, we have $\sigma_{\max}(\bH)\leq 35\sigma_{\max}(\bH_0)/18 < 2\sigma_{\max}(\bH_0)$. In addition,
\begin{align}
\kappa(\bH) &\leq \left(\frac{\sigma_{\max}(\bH_0)+\cuD_2}{\sigma_{\min}(\bH_0)-\cuD_2}\right)\left(1+\frac{4\sqrt{r}(\cuD_1+\cuD_2)}{\sigma_{\min}(\bH_0)-(\cuD_1+\cuD_2)-2\sqrt{r}(\cuD_1+\cuD_2)}\right)\\
&\leq \frac{6\kappa(\bH_0)+1}{5}\left(1+\frac{4}{3}\right) \leq \frac{42\kappa(\bH_0)+7}{15} < \frac{7\kappa(\bH_0)}{2},
\end{align}
and this completes the proof.
\end{proof}
\begin{lemma}
\label{lemma:condition2}
Let $\bX_0 =\bW_0\bH_0\in \reals^{n\times d}$ be such that  $\conv(\bX_0)$ has internal radius at least $\mu >
0$, and $\bX=\bX_0+\bZ$ with $ \max_{i\le n}\|\bZ_{i,.}\|_2 \leq \delta$. If $\bH \in \reals^{r\times d}, \bH_{i,.} \in \aff(\bH_0)$ is 
feasible for problem \eqref{eq:HardNoise} and has  linearly independent rows, then  we have
\begin{align}
\sigma_{\min}(\bH) \geq \sqrt{2}(\mu-2\delta)\, .
\end{align}
\end{lemma}
\begin{proof}
Let
\begin{align}
\bX_{i,.}^\prime = \bPi_{\conv(\bH)}(\bX_{i,\cdot}) \equiv \arg\min_{\bx\in \conv(\bH)}\|\bX_{i,\cdot}-\bx\|_2\,.
\end{align}
Note that since $\bH$ is feasible for problem \eqref{eq:HardNoise} and $\max_{i\le n}\|\bZ_{i,.}\|_2 \leq \delta$
\begin{align}
\|(\bX_0)_{i,.}-\bX^\prime_{i,.}\|_2 \leq \|(\bX_0)_{i,.} - \bX_{i,.}\|_2 + \|\bX_{i,.}-\bX^\prime_{i,.}\|_2 \leq 2\delta.
\end{align}
Therefore, for any $\bx_0 \in \conv(\bX_0)$, writing $\bx_0 = \bX_0^\sT\ba_0$, $\ba_0\in \Delta^n$, we have 
\begin{align}
\cuD(\bx_0, \bX^\prime)^{1/2} &= \min_{\ba\in \Delta^n}\left\|\bX_0^{\sT}\ba_0- \bX^{\prime \sT}\ba \right\|_2 \leq \left\|\bX_0^{\sT}\ba_0 - \bX^{\prime \sT}\ba_0 \right\|_2\\
&\leq \left(\sum_{i=1}^n (a_0)_i \right)\|(\bX_0)_{i,.}-\bX^\prime_{i,.}\|_2 \leq 2\delta.
\label{eq:distx_0convxprime}
\end{align}
Since $\conv(\bX_0)$ has internal radius at least $\mu$, there exists $\bz_0\in\reals^d$, and an orthogonal matrix 
$\bU\in\reals^{d\times r^\prime}$, $r^\prime = r-1$, such that $\bz_0+\bU\Ball_{r^\prime}(\mu)\subseteq\conv(\bX_0)$. 
Hence, for every $\bz\in \reals^{r^\prime}$, $\|\bz\|_2 = 1$
there exists $\ba \in \Delta^n$ such that
\begin{align}
\mu\bU\bz + \bz_0 = \bX_0^\sT\ba.
\end{align}
Therefore, for any unit vector $\bu$ in column space of $\bU$, for the line segment
\begin{align}
\label{eq:lumu}
l_{\bu,\mu} = \left\{\bz: \bz = \bz_0 + \alpha \bu, |\alpha| \leq \mu\right\}\subseteq \conv(\bX_0)\, .
\end{align}
Thus,
\begin{align}
l_{\bu,\mu} \subseteq \bP_{\bu}(\conv(\bX_0))
\end{align}
where $\bP_{\bu}$ is the orthogonal projection onto the line containing $l_{\bu,\mu}$.
Note that using \eqref{eq:distx_0convxprime}, for any $\bx_0\in \conv(\bX_0)$ we have
\begin{align}
\cuD(\bP_{\bu}(\bx_0), \bP_{\bu}(\conv(\bX^\prime)))^{1/2}\leq \cuD(\bx_0,\bX^\prime)^{1/2} \leq 2\delta.
\end{align}
In other words, for any $\bx_0 \in \bP_{\bu}(\conv(\bX_0))$, $D(\bx_0, \bP_{\bu}(\conv(\bX^\prime))) \leq 2\delta$.
Therefore, using \eqref{eq:lumu} for any $\bu$ in column space of $\bU$, we have
\begin{align}
l_{\bu,\mu-2\delta} \subseteq \bP_{u}(\conv(\bX^\prime)).
\end{align}
This implies that 
\begin{align}
\bz_0 + \bU\Ball_{r^\prime}(\mu-2\delta)\subseteq \conv(\bX^\prime) \subseteq \conv(\bH).
\end{align}
Hence, for every $\bz\in \reals^{r^\prime}$, $\|\bz\|_2 = 1$
there exists $\ba \in \Delta^r$ such that
\begin{align}
(\mu-2\delta)\bU\bz + \bz_0 = \bH^\sT\ba.
\end{align}
Note that $\bH^\sT$ has linearly independent columns. Multiplying the previous 
equation by $(\bH^\sT)^\dagger$ the left inverse of $\bH^\sT$, we get
\begin{align}
(\mu-2\delta)(\bH^\sT)^\dagger\bU\bz + (\bH^\sT)^\dagger\bz_0 = \ba.
\end{align}
Let
\begin{align}
&\ba_1 = (\mu-2\delta)(\bH^\sT)^\dagger\bU\bv + (\bH^\sT)^\dagger\bz_0\, ,\\
&\ba_2 = -(\mu-2\delta)(\bH^\sT)^\dagger\bU\bv + (\bH^\sT)^\dagger\bz_0\, ,
\end{align}
where $\bv$ is the right singular vector corresponding to the largest singular value of $(\bH^\sT)^\dagger\bU$. 
Therefore, we have
\begin{align}
&\ba_1 = (\mu - 2\delta)\sigma_{\max}((\bH^\sT)^\dagger\bU)\bv + (\bH^\sT)^\dagger\bz_0,\\
&\ba_2 = -(\mu - 2\delta)\sigma_{\max}((\bH^\sT)^\dagger\bU)\bv + (\bH^\sT)^\dagger\bz_0.
\end{align}
Thus, for $\ba_1, \ba_2 \in \Delta^r$
\begin{align}
\|\ba_1 - \ba_2\|_2 = 2(\mu - 2\delta)\sigma_{\max}((\bH^\sT)^\dagger\bU).
\end{align}
Note that
\begin{align}
\|\ba_1 - \ba_2\|_2 \leq {\sqrt{2}}.
\end{align}
Thus,
\begin{align}
2(\mu - 2\delta)\sigma_{\max}((\bH^\sT)^\dagger\bU) = \frac{2(\mu-2\delta)}{\sigma_{\min}(\bH)}\leq \sqrt{2}.
\end{align}
Hence,
\begin{align}
\sigma_{\min}(\bH) \geq \sqrt{2}(\mu-2\delta).
\end{align}
\end{proof}
The following lemma states an important property of $\hbH$ the optimal solution of problem \eqref{eq:HardNoise}.
\begin{lemma}
\label{lemma:optsol}
If $ \max_{i}\|\bZ_{i,.}\|_2 \leq \delta$ and $\hbH$ is the optimal solution of problem \eqref{eq:HardNoise}, then we have
\begin{align}
\cuD(\hbH,\bX_0)^{1/2}\leq \cuD(\bH_0,\bX_0)^{1/2}+3\delta\sqrt{r}.
\end{align}
\end{lemma}
\begin{proof}
First note that since $\delta \geq \max_{i}\|\bZ_{i,.}\|_2$, we have
\begin{align}
\max_{i\le n}\cuD(\bX_{i,.},\conv(\bH_0))^{1/2} \leq \max_{i\le n}\|\bZ_{i,.}\|_2 \leq \delta.
\end{align}
Hence, $\bH_0$ is a feasible solution for the problem \eqref{eq:HardNoise}. Therefore, we have
\begin{align}
\cuD(\hbH,\bX) \leq \cuD(\bH_0,\bX).\label{eq:Optimality}
\end{align}
Letting $\tilde\balpha_i = \arg\min_{\balpha \in \Delta^n}\|{\hbH_{i,.}}^\sT-\bX^\sT\balpha\|_2$, we have
\begin{align}
\cuD(\hbH,\bX) &= \sum_{i=1}^r\min_{\balpha_i\in \Delta^r}\|{\hbH_{i,.}}^\sT- \bX_0^\sT\balpha_i-\bZ^\sT\balpha_i\|_2^2\\
&= \sum_{i=1}^r\min_{\balpha_i\in \Delta^r}\left(\|{\hbH_{i,.}}^\sT- \bX_0^\sT\balpha_i\|_2^2 - 2\left\langle\bZ^\sT\balpha_i,{\hbH_{i,.}}^\sT-\bX_0^\sT\balpha_i\right\rangle + \|\bZ^\sT\balpha_i\|_2^2\right)\\
&= \sum_{i=1}^r\left(\|{\hbH_{i,.}}^\sT- \bX_0^\sT\tilde\balpha_i\|_2^2 - 2\left\langle\bZ^\sT\tilde\balpha_i,{\hbH_{i,.}}^\sT-\bX_0^\sT\tilde\balpha_i\right\rangle + \|\bZ^\sT\tilde\balpha_i\|_2^2\right).
\end{align}
Using the fact that (by triangle inequality) $\|\bZ^\sT\tilde\balpha_i\|_2 \leq \delta$, we have
\begin{align}
\cuD(\hbH,\bX) &\geq \sum_{i=1}^r \left(\|{\hbH_{i,.}}^\sT- \bX_0^\sT\tilde\balpha_i\|_2^2 - 2\delta\|{\hbH_{i,.}}^\sT - \bX_0^\sT\tilde\balpha_i\|_2\right)\\
&\geq U^2- 2\delta\sqrt{r} U
\end{align}
where $U^2 = \sum_{i=1}^r\|{\hbH_{i,.}}^\sT- \bX_0^\sT\tilde\balpha_i\|_2^2$. Note that $\cuD(\hbH,\bX)\geq 0$ and for 
$U\geq 2\delta \sqrt{r}$, the function $U^2-2\delta\sqrt{r}U$ is increasing. Hence, since 
\begin{align}
U\geq \left(\sum_{i=1}^r\min_{\balpha_i}\|{\hbH_{i,.}}^\sT- \bX_0^\sT\balpha_i\|_2^2\right)^{1/2} = \cuD(\hbH,\bX_0)^{1/2},
\end{align}
we have
\begin{align}
\cuD(\hbH,\bX)\geq (U^2-2\delta\sqrt{r}U)\mathbb I_{U\geq 2\delta\sqrt{r}} \geq \cuD(\hbH,\bX_0)-2\delta\sqrt{r}\cuD(\hbH,\bX_0)^{1/2}.
\end{align}
Therefore,
\begin{align}
\label{eq:1lemmaoptsol}
\cuD(\hbH,\bX)^{1/2} \geq \left(\cuD(\hbH,\bX_0)-2\delta\sqrt{r}\cuD(\hbH,\bX_0)^{1/2}\right)_+^{1/2}\geq \cuD(\hbH,\bX_0)^{1/2}- 2\delta\sqrt{r}.
\end{align}
In addition, 
\begin{align}
\cuD(\bH_0,\bX) &= \sum_{i=1}^r \min_{\balpha_i\in \Delta^{n}}\|(\bH_0)_{i,.} - \bX_0^\sT\balpha_i-\bZ^\sT\balpha_i\|_2^2\\
&\leq \sum_{i=1}^r \min_{\balpha_i\in \Delta^{n}}\left\{\|(\bH_0)_{i,.} - \bX_0^\sT\balpha_i\|_2 + \|\bZ^\sT\balpha_i\|_2\right\}^2\\
&\leq \sum_{i=1}^r \left\{\min_{\balpha_i \in \Delta^n}\|(\bH_0)_{i,.} - \bX_0^\sT\balpha_i\|_2 + \max_{\balpha_i\in \Delta^n}\|\bZ^\sT\balpha_i\|_2\right\}^2\\
&\leq \left\{\left(\sum_{i=1}^r\min_{\balpha_i\in \Delta^n}\|(\bH_0)_{i,.}- \bX_0^\sT\balpha_i\|_2^2\right)^{1/2} + \delta\sqrt{r} \right\}^2\\
&\leq \left(\cuD(\bH_0,\bX_0)^{1/2} + \delta\sqrt{r}\right)^2.
\end{align}
Hence, 
\begin{align}
\label{eq:2lemmaoptsol}
\cuD(\bH_0,\bX)^{1/2} \leq \cuD(\bH_0,\bX_0)^{1/2} + \delta\sqrt{r}.
\end{align}
Combining equations \eqref{eq:1lemmaoptsol}, \eqref{eq:2lemmaoptsol}, and \eqref{eq:Optimality}, we get
\begin{align}
\cuD(\hbH,\bX_0)^{1/2}\leq \cuD(\bH_0,\bX_0)^{1/2}+3\delta\sqrt{r}.
\end{align}
This completes the proof of lemma.
\end{proof}

\begin{lemma}
\label{lemma:boundc}
Let $\bX_0$ be such that the uniqueness assumption holds with parameter $\alpha>0$, and $\conv(\bX_0)$ has internal radius 
at least $\mu>0$.  In particular, we have $\bz_0+\bU\Ball_{r-1}(\mu)\subseteq \conv(\bX_0)$ for  $\bz_0\in\reals^d$, and
an orthogonal matrix $\bU\in\reals^{d\times (r-1)}$.
Finally assume  $\max_{i\le n}\|\bZ_{i,.}\|_2\leq \delta$.
Then for $\hbH$
the optimal solution of problem \eqref{eq:HardNoise}, we have
\small
\begin{align}
\alpha(\cuD(\hbH, \bH_0)^{1/2} + \cuD(\bH_0,\hbH)^{1/2}) \leq 2(1+2\alpha)\left[r^{3/2}\delta\kappa(\bP_0(\hbH)) +\frac{\delta\sqrt{r}}{\mu}\sigma_{\max}(\hbH - \one\bz_0^\sT)\right] +3\delta\sqrt{r}
\end{align}
\normalsize
where $\bP_0:\reals^d\rightarrow \reals^d$ is the orthogonal projector onto $\aff(\bH_0)$ (in particular, $\bP_0$ is an affine map).
\end{lemma}
\begin{proof}
Let $\tbH$ be such that $\conv(\bX_0) \subseteq \conv(\tbH)$. The uniqueness assumption implies 
\begin{align}
\cuD(\tbH,\bX_0)^{1/2} \geq \cuD(\bH_0,\bX_0)^{1/2} + \alpha\big(\cuD(\tbH,\bH_0)^{1/2} +  \cuD(\bH_0,\tbH)^{1/2}\big).
\end{align}
Note that Lemma \ref{lemma:optsol} implies
\begin{align}
\cuD(\hbH,\bX_0)^{1/2}\leq \cuD(\bH_0,\bX_0)^{1/2}+3\delta\sqrt{r}.
\end{align} 
Therefore,
\begin{align}
\cuD(\tbH,\bX_0) ^{1/2}\geq \cuD(\hbH, \bX_0)^{1/2} - 3\delta\sqrt{r} + \alpha\big(\cuD(\tbH,\bH_0)^{1/2} + \cuD(\bH_0,\tbH)^{1/2}\big)\, .
\end{align}
Hence,
\begin{align}
\label{eq:alphaCless}
\alpha\big(\cuD(\tbH,\bH_0)^{1/2} +  \cuD(\bH_0,\tbH)^{1/2}\big) \leq \cuD(\tbH,\bX_0)^{1/2} - \cuD(\hbH,\bX_0)^{1/2} + 3\delta\sqrt{r}.
\end{align}
In addition, for a convex set $S$, by triangle inequality we have
\begin{align}
&\left[\sum_{i=1}^n\cuD({\hbH_{i,.}},S)\right]^{1/2} - \left[\sum_{i=1}^n\cuD(\tbH_{i,.},S)\right]^{1/2}
\leq \left[\sum_{i=1}^n\left(\cuD({\hbH_{i,.}},S)^{1/2}-\cuD(\tbH_{i,.},S)^{1/2}\right)^2\right]^{1/2}
\end{align}
Therefore, using
\begin{align}
| \cuD(\tbH_{i,.},S)^{1/2} - \cuD({\hbH_{i,.}},S)^{1/2}| \leq \|\tbH_{i,.} - \hbH_{i,.}\|_2
\end{align}
we have
\begin{align}
\left[\sum_{i=1}^n\cuD({\hbH_{i,.}},S)\right]^{1/2} - \left[\sum_{i=1}^n\cuD(\tbH_{i,.},S)\right]^{1/2} \leq \left[\sum_{i=1}^{n}\|\hbH_{i,.} - \tbH_{i,.}\|_2^2\right]^{1/2} = \|\hbH-\tbH\|_F.
\end{align}
Hence,
\begin{align}
\label{eq:CHtildeX0-CHhatX0}
|\cuD(\tbH,\bX_0) ^{1/2} - \cuD(\hbH,\bX_0) ^{1/2}| \leq \|\tbH - \hbH\|_F
\end{align}
and
\begin{align}
\label{eq:CHtildeH0-CHhatH0}
|\cuD(\tbH,\bH_0)^{1/2} - \cuD(\hbH,\bH_0)^{1/2}| \leq \|\tbH - \hbH\|_F.
\end{align}
In addition, similarly to the proof of Lemma \ref{lemma:optsol}, we can write
\begin{align}
\cuD(\bH_0,\tbH) &= \sum_{i=1}^r \min_{\balpha_i\in \Delta^{r}}\|(\bH_0)_{i,.} - {\tbH}^\sT\balpha_i\|_2^2\\
&=\sum_{i=1}^r \min_{\balpha_i\in \Delta^{r}}\|(\bH_0)_{i,.} - {\hbH}^\sT\balpha_i-(\hbH-\tbH)^\sT\balpha_i\|_2^2\\
&\leq \sum_{i=1}^r \min_{\balpha_i\in \Delta^{r}}\left\{\|(\bH_0)_{i,.} - {\hbH}^\sT\balpha_i\|_2 + \|(\hbH-\tbH)^\sT\balpha_i\|_2\right\}^2\\
&\leq \sum_{i=1}^r \left\{\min_{\balpha\in \Delta^r}\|(\bH_0)_{i,.} - {\hbH}^\sT\balpha\|_2 + \max_{\balpha\in \Delta^r}\|(\hbH-\tbH)^\sT\balpha\|_2\right\}^2\\
&\leq \left\{\left(\sum_{i=1}^r\min_{\balpha\in \Delta^r}\|(\bH_0)_{i,.} - {\hbH}^\sT\balpha_i\|_2^2\right)^{1/2} + \sqrt{r}\max_{i\in [r]}\|\hbH_{i,.}-\tbH_{i,.}\|_2 \right\}^2\\
&\leq \left(\cuD(\bH_0,\hbH)^{1/2} + \sqrt{r}\max_{i \in [r]}\|\hbH_{i,.}-\tbH_{i,.}\|_2\right)^2.
\end{align}
Thus,
\begin{align}
\label{eq:CHhatH0-CH0Htilde}
|\cuD(\bH_0,\tbH)^{1/2} - \cuD(\bH_0,\hbH)^{1/2}| \leq \sqrt{r} \max_{i\in [r]} \|\tbH_{i,.} - \hbH_{i,.}\|_2
\end{align}
Therefore, combining \eqref{eq:alphaCless}, 
\eqref{eq:CHtildeX0-CHhatX0}, \eqref{eq:CHtildeH0-CHhatH0}, \eqref{eq:CHhatH0-CH0Htilde}, we get
\begin{align}
\label{eq:Cwhatw0}
\alpha\big(\cuD(\hbH, \bH_0)^{1/2} + \cuD(\bH_0,\hbH)^{1/2}\big)  \leq (1+\alpha)\|\tbH-\hbH\|_F + \alpha\sqrt{r}\max_{i\in [r]}\|\tbH_{i,.} - \hbH_{i,.}\|_2 + 3\delta\sqrt{r}.
\end{align}
Now, we would like to bound the terms $\|\tbH - \hbH\|_F$, $\max_{i\in [r]}\|\tbH_{i,.} - \hbH_{i,.}\|_2$. 
Note that using the fact that $\hbH$ is feasible for Problem \eqref{eq:HardNoise}, we have
\begin{align}
\cuD(\bX_{i,.}, \hbH) \leq \delta^2\, .
\end{align}
Thus,
\begin{align}
\cuD((\bX_0)_{i,.},\hbH)^{1/2} \leq \|\bX_{i,.} - (\bX_0)_{i,.}\|_2 + \cuD(\bX_{i,.},\hbH)^{1/2} \leq 2\delta.
\end{align}
In addition, we know that $(\bX_0)_{i,.} \in \aff(\bH_0)$, where $\aff(\bH_0)$ is a $r-1$ dimensional affine subspace. 
Therefore, $\conv(\bX_0) \subseteq \aff(\bH_0)$ and, by convexity of $\Ball_{d}(2\delta,\hbH)$, we get
\begin{align}
\label{eq:Xi0inballcapaff}
\conv(\bX_0) \subseteq \Ball_{d}(2\delta,\hbH)\cap \aff(\bH_0).
\end{align}
First consider the case in which $\hbH_{i,.} \in \aff(\bH_0)$. for all $i\in\{1,2,\dots,r\}$.
By a perturbation argument, we can assume that the rows of $\hbH$ are linearly independent, and hence $\aff(\hbH) = \aff(\bH_0)$.
Consider $\tilde \bQ\in \reals^{r\times d}$ defined by
\begin{align}
&\widetilde Q_{ii} = 1+\xi, \quad\quad \text{if}\;\; i = j\in \{1,2,\dots,r\},\\
&\widetilde Q_{ij} = -\frac{\xi}{r-1},  \;\quad \text{if}\;\; i\neq j\in \{1,2,\dots,r\},\\
&\widetilde Q_{ij} = 0, \quad\quad\quad\quad \text{if}\;\; j\in\{r+1,r+2,\dots,d\}
\end{align}
where $\xi = 2r\delta_0$. Note that for every $\by\in \Ball_d(2\delta_0;\bE_{r,d}) \cap \aff(\bE_{r,d})$, 
we have $\cuD(\by,\bE_{r,d})^{1/2} \leq 2\delta_0$. In addition, since $\by \in \aff(\bE_{r,d})$, $\langle \by,\one\rangle = 1$. Hence, for 
$\by\in \Ball_d(2\delta_0;\bE_{r,d}) \cap \aff(\bE_{r,d})$, we can write
\begin{align}
\by =\bpi + \bx
\end{align}
where $\bpi \in \conv(\bE_{r,d})$, $\bx\in \reals^d$, $\langle\one,\bx\rangle = 0$, $\|\bx\|_2 \leq 2\delta_0$.
It is easy to check that for this $\by$ we have
\begin{align}
\by = \sum_{i=1}^{r}\beta_i\widetilde\bQ_{i,.}
\end{align}
where $\bbeta\in \reals^r$ is such that for $i = 1,2,\dots,r$,
\begin{align}
\beta_i = \frac{r-1}{r-1+\xi r}(\pi_i + x_i) + \frac{\xi}{r-1+\xi r}.
\end{align}
Further, note that since $\bpi \in \conv(\bE_{r,d})$, $\pi_i\geq 0$ and $x_i \geq -\|\bx\|_2 \geq -2\delta_0$, we have
$\pi_i + x_i \geq -2\delta_0$. Hence, 
for $i\in \{1,2,\dots,r\}$,
\begin{align}
\beta_i \geq \frac{-2\delta_0(r-1) + \xi}{r-1+\xi r} = \frac{2\delta_0}{r-1+\xi r} \geq 0.
\end{align}
In addition,
\begin{align}
\sum_{i = 1}^r \beta_i = \frac{r\xi}{r-1+\xi r} + \frac{r-1}{r-1+\xi r}\left(\sum_{i=1}^r(\pi_i + x_i) \right)= 1.
\end{align}
Therefore, every $\by\in \Ball_d(2\delta_0;\bE_{r,d}) \cap \aff(\bE_{r,d})$ can be written as a convex
combination of the rows of $\widetilde\bQ$.
Hence, 
\begin{align}
\Ball_{d}(2\delta_0;\bE_{r,d})\cap{\aff}(\bE_{r,d})\subseteq\conv(\widetilde\bQ).
\end{align}
Let $\hbH = \bE_{r,d}\bM$, $\bM\in \reals^{d\times d}$. 
Since $\aff(\hbH) = \aff(\bH_0)$, by taking
$\tbH = \widetilde\bQ\bM$, we have
\begin{align}
\conv(\tbH) &\supseteq \left[\cup_{\bx\in \conv(\bE_{r,d})}\bM^\sT \Ball_{d}(2\delta_0;\bx)\right] \cap{\aff} (\hbH)\\
&\supseteq \left[\cup_{\bx \in \conv(\hbH)}\Ball_{d}(2\delta_0\sigma_{\min}(\bM);\bx)\right]\cap{\aff} (\hbH)\\
&\supseteq \Ball_{d}(2\delta;\hbH)\cap{\aff}(\bH_0),
\end{align}
provided that $\delta_0 = \delta/\sigma_{\min}(\bM) = \delta/\sigma_{\min}(\hbH)$. Hence, using \eqref{eq:Xi0inballcapaff} for this $\delta_0$, 
$\conv(\bX_0)\subseteq \conv(\tbH)$. Note that for $\widetilde\bQ$, we have
$\|\widetilde\bQ_{i,.} - \be_i\|_2 \leq 2r\delta_0$. Thus,
\begin{align}
\|\widetilde\bQ- \bE_{r,d}\|_F \leq 2r^{3/2}\delta_0.
\end{align}
Therefore, there exists $\tbH\in \reals^{r\times d}$ such that $\conv(\bX_0) \subseteq \conv(\tbH)$ and
\begin{align*}
&\|\tbH - \hbH\|_F = \|(\widetilde\bQ - \bE_{r,d})\bM\|_F\leq2r^{3/2}\delta_0\sigma_{\max}(\bM) = 2r^{3/2}\delta_0\sigma_{\max}(\hbH) = 2r^{3/2}\delta\kappa(\hbH),\\
&\max_{i\in [r]}\|\tbH_{i,.} - \hbH_{i,.}\|_2 = \max_{i\in [r]} \|(\widetilde\bQ_{i,.} - \be_i)\bM\|_2 \leq 2r\delta_0\sigma_{\max}(\bM) = 2r\delta_0\sigma_{\max}(\hbH) = 2r\delta\kappa(\hbH).
\end{align*}
Now consider the general case in which $\aff(\hbH) \neq \aff(\bH_0)$. Let $\bH^\prime\in \reals^{r\times d}$
be such that $\bH^\prime_{i,.}$ is the 
projection of $\hbH_{i,.}$ onto $\aff(\bH_0)$. Assuming that the 
rows of $\bH^\prime$ are linearly independent, $\aff(\bH^\prime) = \aff(\bH_0)$.
Note that since $\conv(\bX_0) \in \aff(\bH_0)$, for every point $\bx \in \conv(\bX_0)$, 
$\cuD(\bx,\bH^\prime)^{1/2}\leq \cuD(\bx,\hbH)^{1/2} \leq 2\delta$. Thus,
\begin{align}
(\bX_0)_{i,.} \in \Ball_d(2\delta, \bH^\prime)\cap {\aff}(\bH^\prime).
\end{align}
Therefore, using the above argument for the case where $\aff(\hbH) = \aff(\bH_0)$, we can find $\tbH$ such that $\conv(\bX_0)\subseteq \conv(\tbH)$ and
\begin{align}
&\|\tbH - \bH^\prime\|_F \leq 2r^{3/2}\delta\kappa(\bH^\prime),\\
&\max_{i\in [r]}\|\tbH_{i,.} - \bH_{i,.}^\prime\|_2 \leq 2r\delta\kappa(\bH^\prime).  
\end{align}
Hence, for every $i=1,2,\dots,r$,
\begin{align}
\begin{split}
\label{eq:w*-wtilde1}
\|\tbH_{i,.} - \hbH_{i,.}\|_2 & \leq  \|\tbH_{i,.} -\bH^\prime_{i,.}\|_2 + \|\bH^\prime_{i,.} - \hbH_{i,.}\|_2  \\
& \leq 2r\delta\kappa(\bH^\prime) + \|\bP_0(\hbH_{i,.}) - \hbH_{i,.}\|_2\\
\end{split}
\end{align}
where $\bP_0$ orthogonal projection onto $\aff(\bH_0)$. We next use the assumption on the internal radius of $\conv(\bX_0)$ to upper bound the term 
$\|\bP_0(\hbH_{i,.}) - \hbH_{i,.}\|_2$. Note that since 
$\conv(\bX_0) \subseteq \Ball_d(2\delta,\hbH)$, letting $\bar\bH = \hbH - \one\bz_0^\sT$,
for some orthogonal matrix $\bU\in\reals^{d\times r^\prime}$, $r^\prime=r-1$, we have
\begin{align}
\max_{\|\bz\|_2\leq \mu}\min_{\langle \ba,\one\rangle= 1, \ba\geq 0}\|\bU\bz - \bar\bH^\sT\ba \|_2^2 &= 
\max_{\|\bz\|_2\leq \mu}\min_{\langle \ba,\one\rangle= 1, \ba\geq 0}\|\bU\bz - (\hbH - \one\bz_0^\sT)^\sT\ba \|_2^2\\ &\leq \max_{\|\bz\|_2\leq \mu}\min_{\langle \ba,\one\rangle= 1, \ba\geq 0}\|\bU\bz + \bz_0 - {\hbH}^\sT\ba\|_2^2
\leq 4\delta^2.
\end{align}
Now, using Cauchy-Schwarz inequality we can write
\begin{align}
\label{eq:maxoverzminovera}
\max_{\|\bz\|_2\leq \mu}\min_{\|\ba\|_2\leq 1}\|\bU\bz - {\bar\bH}^\sT\ba\|_2^2 \leq \max_{\|\bz\|_2\leq \mu}\min_{\langle \ba,\one\rangle= 1, \ba\geq 0}\|\bU\bz - \bar\bH^\sT\ba \|_2^2 \leq 4\delta^2.
\end{align}
Note that,
\begin{align}
\min_{\|\ba\|_2\leq 1}\|\bU\bz - \bar\bH^{\sT}\ba\|_2^2 &=\max_{\rho\geq 0}\min_{\ba}\left\{ \|\bz\|_2^2 - 2\left\langle\bz,\bU^\sT{\bar\bH}^\sT\ba\right\rangle + \left\langle\ba,({\bar\bH}{\bar\bH}^\sT+\rho\Id)\ba\right\rangle - \rho\right\}\\
&= \max_{\rho\geq 0} \left\{\|\bz\|_2^2 - \left\langle{\bar\bH}\bU\bz,({\bar\bH}{\bar\bH}^\sT+\rho\Id)^{-1}{\bar\bH}\bU\bz\right\rangle - \rho\right\}
\end{align}
Hence, using \eqref{eq:maxoverzminovera}
\begin{align}
\mu^2\max_{\rho\geq 0} \Big\{\lambda_{\max}(\Id - \bU^\sT{\bar\bH}^\sT(\bar\bH{\bar\bH}^\sT+\rho\Id)^{-1}\bar\bH\bU)-\rho \leq 4\delta^2\Big\}.
\end{align}
In particular, for $\rho = 0$ we get
\begin{align}
\mu^2 \lambda_{\max}(\Id - \bU^\sT{\bar\bH}^\sT(\bar\bH{\bar\bH}^\sT)^{-1}\bar\bH\bU) \leq 4\delta^2.
\end{align}
Taking $\bar\bH = \tilde\bU\bSigma\tilde\bV^\sT$, the singular value decomposition of $\bar\bH$, we have
$\sigma_{\max}(\bar\bH) = \sigma_{\max}(\hbH - \one\bz_0^\sT) = \max_{i}\Sigma_{ii}$. Letting
$\bU^\sT\tilde\bV = \bQ$, we get
\begin{align}
\max_{\rho\geq 0}\lambda_{\max}\left(\Id - \bQ\bQ^\sT\right)\leq \frac{4\delta^2}{\mu^2}.
\end{align}
Letting $q = \sigma_{\min}(\bQ)$, this results in
\begin{align}
\label{eq:assumpresult}
1-q^2 \leq \frac{4\delta^2}{\mu^2}.
\end{align}
In addition, note that, by the internal radius assumption, for any $\bz\in \reals^{r^\prime}$, $\bz_0 + \bU\bz \in \aff(\bH_0)$.
Further, since $\bz_0 \in \aff(\bH_0)$,
\begin{align}
\max_{i \in [r]}\|\bP_0(\hbH_{i,.}) - \hbH_{i,.}\|_2 &= \max_{i \in [r]}\|\bP_\bU(\bar\bH_{i,.}) - \bar\bH_{i,.}\|_2\\
&\leq \max_{\|\ba\|_2\leq 1}\|\bP_\bU(\bar\bH^\sT\ba) - \bar\bH^\sT\ba\|_2\\
&\leq\max_{\|\ba\|_2\leq 1}\|\bP_\bU(\bar\bH^\sT\ba) - \bar\bH^\sT\ba\|_2\\
&\leq\max_{\|\ba\|_2\leq 1}\min_{\bz}\|\bU\bz-{\bar\bH}^\sT\ba\|_2^2
\end{align}
where $\bP_\bU$ is the projector onto the column space of $\bU$. Note that,
\begin{align}
\max_{\|\ba\|_2\leq 1}\min_{\bz}\|\bU\bz-{\bar\bH}^\sT\ba\|_2^2 &= \max_{\|\ba\|_2\leq 1}\left\{-\left\langle\ba,\bar\bH\bU\bU^\sT{\bar\bH}^\sT\ba\right\rangle + \left\langle\ba,\bar\bH{\bar\bH}^\sT\ba\right\rangle\right\}\\
&= \lambda_{\max}(\bar\bH{\bar\bH}^\sT - \bar\bH\bU\bU^\sT{\bar\bH}^\sT)\\
&= \lambda_{\max}(\bSigma(\Id-\bQ^\sT\bQ)\bSigma)\\
&\leq \sigma_{\max}(\bar\bH)^2\lambda_{\max}(\Id-\bQ^\sT\bQ) \\
&\leq \sigma_{\max}(\bar\bH)^2(1-q^2)\leq \frac{4\sigma_{\max}(\bar\bH)^2\delta^2}{\mu^2}
\end{align}
where the last inequality follows from \eqref{eq:assumpresult}. This results in
\begin{align}
\max_{i \in [r]}\|\bP_0(\hbH_{i,.}) - \hbH_{i,.}\|_2 \leq \frac{2\sigma_{\max}(\bar\bH)\delta}{\mu} &= 
\frac{2\sigma_{\max}(\hbH - \one\bz_0^\sT)\delta}{\mu}
\end{align}
Therefore, $\|\bP_0(\hbH) - \hbH\|_F \leq 2\sigma_{\max}(\hbH - \one\bz_0^\sT)\delta\sqrt{r}/\mu$. Hence, using \eqref{eq:w*-wtilde1} we get
\begin{align}
&\max_{i\in[r]}\|\hbH_{i,.} - \tbH_{i,.}\|_2 \leq 2r\delta\kappa(\bP_0(\hbH)) + \frac{2\sigma_{\max}(\hbH - \one\bz_0^\sT)\delta}{\mu},\\
&\|\hbH - \tbH\|_F \leq 2r^{3/2}\delta\kappa(\bP_0(\hbH)) + \frac{2\sigma_{\max}(\hbH - \one\bz_0^\sT)\delta\sqrt{r}}{\mu}.
\end{align}
Replacing this in \eqref{eq:Cwhatw0} completes the proof.
\end{proof}

\subsubsection{Proof of Theorem \ref{thm:Robust2}}

For simplicity, let $\cuD = \alpha(\cuD(\hbH,\bH_0)^{1/2} + \cuD(\bH_0,\hbH)^{1/2})$. First note that under the assumption 
of Theorem \ref{thm:Robust2} we have
\begin{align}
\bz_0 + \bU\Ball_{r^\prime}(\mu)\subseteq \conv(\bX_0) \subseteq \conv(\bH_0).
\end{align}
Therefore, using  Lemma \ref{lemma:condition2} with $\bH=\bH_0$ and $\delta=0$, we have
\begin{align}
\label{eq:sigmaminH0overmu}
\mu\sqrt{2} \leq \sigma_{\min}(\bH_0) \leq \sigma_{\max}(\bH_0).
\end{align}
In addition, since $\bz_0 \in \conv(\bH_0)$ we have $\bz_0 = \bH_0^\sT\balpha_0$ for some $\balpha_0 \in \Delta^r$.
Therefore,
\begin{align}\label{eq:sigmamaxH0normz0}
\|\bz_0\|_2 \leq \sigma_{\max}(\bH_0)\|\balpha_0\|_2 \leq \sigma_{\max}(\bH_0).
\end{align}
Note that
\begin{align}
\sigma_{\max}(\hbH - \one\bz_0^\sT) \leq \sigma_{\max}(\hbH) + \sigma_{\max}(\one\bz_0^\sT) = 
\sigma_{\max}(\hbH) + \sqrt{r}\|\bz_0\|_2.
\end{align}
Therefore, using Lemma \ref{lemma:boundc} we have
\begin{align}
\label{eq:Clessthan2(1+alpha1+alpha2)}
\cuD\leq 2(1+2\alpha) \left(r^{3/2}\delta\kappa(\bP_0(\hbH)) + \frac{\sigma_{\max}(\hbH)\delta r^{1/2}}{\mu} + \frac{r\delta\|\bz_0\|_2}{\mu}\right) + 3\delta r^{1/2}.
\end{align}
In addition, Lemma \ref{lemma:dandc} implies that 
\begin{align}
\label{eq:DH0HhatCH0Hhat}
\cuL(\bH_0,\hbH)^{1/2} \leq \frac{1}{\alpha}\max\left\{{(1+\sqrt{2})\sqrt{r}},{\sqrt{2}\kappa(\bH_0)}\right\}\cuD\, .
\end{align}
Further, let $\bP_0$ denote the orthogonal projector on $\aff(\bH_0)$. Hence, $\bP_0$ is a non-expansive 
mapping: for $\bx, \by \in \reals^d$, $D(\bP_0(\bx),\bP_0(\by))\le D(\bx,\by)$. Therefore, 
since $\conv(\bH_0) \subset \aff(\bH_0)$, for any $\bh \in \reals^d$
\begin{align}
\cuD (\bP_0(\bh), \bH_0) \le D(\bP_0(\bh), \bP_0(\bPi_{\conv(\bH_0)}(\bh))) \le D(\bh, \bPi_{\conv(\bH_0)}(\bh)) = \cuD(\bh, \bH_0).
\end{align}
Therefore, 
\begin{align}\label{eq:cudprojection}
\cuD(\bP_0(\hbH),\bH_0) \le \cuD(\hbH,\bH_0).
\end{align}
First consider the case in which
\begin{align}
\label{eq:conddeltamain}
\delta \leq \frac{\alpha\mu}{30\, r^{3/2}}.
\end{align}
Note that in this case $\delta \leq \mu/2$. Hence, using Lemma \ref{lemma:condition} to 
upper bound $\sigma_{\max}(\hbH)$, $\sigma_{\max}(\bP_0(\hbH))$ and Lemma \ref{lemma:condition2}
to lower bound $\sigma_{\min}(\bP_0(\hbH))$, by \eqref{eq:cudprojection},
we get
\begin{align}
&\sigma_{\max}(\hbH) \leq \cuD(\hbH,\bH_0)^{1/2} + r^{1/2}\sigma_{\max}(\bH_0) \leq \frac{\cuD}{\alpha} + r^{1/2}\sigma_{\max}(\bH_0),\\
&\kappa(\bP_0(\hbH))= \frac{\sigma_{\max}(\bP_0(\hbH))}{\sigma_{\min}(\bP_0(\hbH))}\leq \frac{\cuD(\bP_0(\hbH),\bH_0)^{1/2} + r^{1/2}\sigma_{\max}(\bH_0)}{\sqrt{2}(\mu - 2\delta)} \nonumber\\
&\;\;\;\;\;\;\;\;\;\;\;\;\;\;\;\;\leq\frac{\cuD(\hbH,\bH_0)^{1/2} + r^{1/2}\sigma_{\max}(\bH_0)}{\sqrt{2}(\mu - 2\delta)}\leq \frac{\cuD}{\alpha(\mu - 2\delta)\sqrt{2}} + \frac{r^{1/2}\sigma_{\max}(\bH_0)}{(\mu-2\delta)\sqrt{2}}.
\end{align}

Replacing these in \eqref{eq:Clessthan2(1+alpha1+alpha2)} we have
\begin{align}
\cuD \leq 2(1+2\alpha)\left[\frac{r^{3/2}\cuD\delta}{\alpha(\mu-2\delta)\sqrt{2}} + \frac{r^2\sigma_{\max}(\bH_0)\delta}{(\mu-2\delta)\sqrt{2}}+\frac{\cuD r^{1/2}\delta}{\alpha\mu} + \frac{r\sigma_{\max}(\bH_0)\delta}{\mu}+ \frac{r\|\bz_0\|_2\delta}{\mu}\right] + 3\delta\sqrt{r}.
\end{align}
Therefore, 
\begin{align}
\cuD&\left[1-\frac{\sqrt{2}(1+2\alpha)r^{3/2}\delta}{\alpha(\mu-2\delta)}-\frac{2(1+2\alpha) r^{1/2}\delta}{\alpha\mu}\right]\nonumber\\
&\;\;\;\;\;\;\;\leq 2(1+2\alpha)\left[\frac{r^2\sigma_{\max}(\bH_0)\delta}{(\mu-2\delta)\sqrt{2}} + \frac{r\sigma_{\max}(\bH_0)\delta}{\mu} + \frac{r\|\bz_0\|_2\delta}{\mu}\right]+3\delta\sqrt{r}
\end{align}
Notice that condition \eqref{eq:conddeltamain} implies that $\mu-2\delta \geq \mu/2$ and
\begin{align}
\frac{\sqrt{2}(1+2\alpha)r^{3/2}\delta}{\alpha(\mu-2\delta)}+\frac{2(1+2\alpha) r^{1/2}\delta}{\alpha\mu} \leq \frac{1}{2}. 
\end{align}
Using the previous two equations, under condition \eqref{eq:conddeltamain} we have
\begin{align}
\cuD&\leq \frac{4(1+2\alpha)r\delta}{\mu}\left[\frac{5r\sigma_{\max}(\bH_0)}{2} + \|\bz_0\|_2\right] +3\delta\sqrt{r}\nonumber\\
&\leq\frac{4(1+2\alpha)r^2}{\mu}\left[\frac{5\sigma_{\max}(\bH_0)}{2}+\frac{\|\bz_0\|_2}{r} + \frac{3\mu}{4(1+2\alpha)r^{3/2}}\right]\delta.
\label{eq:cdeltamain}
\end{align}
Combining this with \eqref{eq:DH0HhatCH0Hhat}, and using the fact that $1+2\alpha \leq 3$, we have under condition \eqref{eq:conddeltamain}
\begin{align}
\cuL(\bH_0,\hbH)^{1/2}&\le  \frac{12r^2}{\mu\alpha}\left(\frac{5\sigma_{\max}(\bH_0)}{2}+\frac{\|\bz_0\|_2}{r} + \frac{3\mu}{4(1+2\alpha)r^{3/2}}\right)\max\left\{(1+\sqrt{2})\sqrt{r}, \sqrt{2}\kappa(\bH_0)\right\}\delta\\
&\le \frac{29\sigma_{\max}(\bH_0)r^{5/2}}{\alpha\mu}\max\left\{1,\frac{\kappa(\bH_0)}{\sqrt{r}}\right\}\left(\frac{5}{2}+\frac{\|\bz_0\|_2}{r\sigma_{\max}(\bH_0)} + \frac{3\mu}{4(1+2\alpha)r^{3/2}\sigma_{\max}(\bH_0)}\right)\delta.
\end{align}
Note that using \eqref{eq:sigmaminH0overmu}, \eqref{eq:sigmamaxH0normz0} and since $\alpha \ge 0$
\begin{align}
\frac{\|\bz_0\|_2}{r\sigma_{\max}(\bH_0)} \leq 1, \quad\quad\quad \frac{3\mu}{4(1+2\alpha)r^{3/2}\sigma_{\max}(\bH_0)} \le \frac{3}{4\sqrt{2}}.
\end{align}
Therefore,
\begin{align}
\cuL(\bH_0,\hbH)^{1/2} \leq \frac{120\sigma_{\max}(\bH_0)r^{5/2}}{\alpha\mu}\max\left\{1,\frac{\kappa(\bH_0)}{\sqrt{r}}\right\}\delta.
\end{align}
Thus,
\begin{align}
\cuL(\bH_0,\hbH)\le \frac{C_*^2\,  r^{5}}{\alpha^2} \max_{i\le n} \|\bZ_{i,\cdot}\|^2_2\, ,
\end{align}
where $C_*$ is defined in Theorem \ref{thm:Robust2}. 

Next, consider the case in which
\begin{align}
\label{eq:conddelta2}
\delta = \max_{i\le n} \|\bZ_{i,\cdot}\|_2\le \frac{\alpha\mu}{330\kappa(\bH_0)r^{5/2}},
\end{align}
Note that using \eqref{eq:sigmaminH0overmu},\eqref{eq:sigmamaxH0normz0}
and since $1+2\alpha \leq 3$, this condition on $\delta$ implies that
\begin{align}
\delta \leq \frac{\alpha\mu\sigma_{\min}(\bH_0)}{12r(1+2\alpha)(5r^{3/2}\sigma_{\max}(\bH_0)+2\|\bz_0\|_2r^{1/2}+3\mu)}.
\end{align}
In particular, condition \eqref{eq:conddeltamain} holds. Hence, using equation \eqref{eq:cdeltamain} we get
\begin{align}\label{eq:cudineqcase2}
\cuD \leq \frac{4(1+2\alpha)r^2}{\mu}\left[\frac{5\sigma_{\max}(\bH_0)}{2}+\frac{\|\bz_0\|_2}{r} + \frac{3\mu}{4(1+2\alpha)r^{3/2}}\right]\delta \leq \frac{\alpha\sigma_{\min}(\bH_0)}{6\sqrt{r}}.
\end{align}
Further, note that since $\bP_0$ is a projection onto an affine subspace, for $\bx\in \reals^d$,
$\bP_0(\bx) = \widetilde{\bP_0} \bx + \bx_0$ for some $\widetilde{\bP_0} \in \reals^{d \times d}, \bx_0 \in \reals^d$. Hence, 
for any $\bpi \in \Delta^r$, $\bh = \hbH^\sT \bpi \in \conv(\hbH)$, we have
\begin{align}
\bP_0(\bh) = \widetilde{\bP_0}\bh + \bx_0 = \widetilde{\bP_0}\hbH^\sT\bpi + \bx_0 = \sum_{i=1}^r \pi_i\left(\widetilde{\bP_0}\hbH^\sT {\boldsymbol{e}}_i + \bx_0\right) = \sum_{i= 1}^r \pi_i \bP_0(\widehat{\bh_i}) \in \conv(\bP_0(\hbH))
\end{align}
where ${\boldsymbol{e}}_i$ is the $i$'th standard unit vector. Hence, 
\begin{align}
\bP_0(\conv(\hbH)) \subseteq \conv (\bP_0(\hbH)).
\end{align}
Thus, for $\bh_0 \in \reals^d$ an arbitrary row of $\bH_0$, we have
\begin{align}
\cuD(\bh_0,\bP_0(\hbH)) = D(\bh_0,\conv(\bP_0(\hbH))) \leq D(\bh_0, \bP_0(\conv(\hbH))) \leq D(\bh_0, \bP_0(\bPi_{\conv(\hbH)}(\bh_0))).
\end{align}
In addition, using non-expansivity of $\bP_0$, we have
\begin{align}
D(\bh_0, \bP_0(\bPi_{\conv(\hbH)}(\bh_0))) \leq D(\bh_0, \bPi_{\conv(\hbH)}(\bh_0)) = D(\bh_0, \conv(\hbH)) = \cuD(\bh_0,\hbH).
\end{align}
This implies that
\begin{align}\label{eq:cudprojection2}
\cuD(\bH_0,\bP_0(\hbH))\le \cuD(\bH_0,\hbH).
\end{align}
Therefore, using \eqref{eq:cudprojection}, \eqref{eq:cudprojection2} and \eqref{eq:cudineqcase2} we get 
\begin{align}
\cuD(\bH_0,\bP_0(\hbH))^{1/2} + \cuD(\bP_0(\hbH),\bH_0)^{1/2} \leq \cuD(\bH_0,\hbH)^{1/2} + \cuD(\hbH,\bH_0)^{1/2} \leq \frac{\cuD}{\alpha}\leq \frac{\sigma_{\min}(\bH_0)}{6\sqrt{r}}.
\end{align}
Hence, in this case Lemma \ref{lemma:condition} implies that
\begin{align}
&\sigma_{\max}(\hbH) \leq 2\sigma_{\max}(\bH_0),\\
&\kappa(\bP_0(\hbH)) \leq \frac{7\kappa(\bH_0)}{2}.
\end{align}
Replacing this in \eqref{eq:Clessthan2(1+alpha1+alpha2)}, we have
\begin{align}
\label{eq:hdeltadef}
\cuD &\leq (1+2\alpha)r^{1/2}\left(7r\delta\kappa(\bH_0) + \frac{4\sigma_{\max}(\bH_0)\delta + 2\sqrt{r}\|\bz_0\|_2\delta}{\mu}\right) + 3\delta r^{1/2}\\
&\leq 3\delta\sqrt{r}\left(8r\kappa(\bH_0)+ \frac{4\sigma_{\max}(\bH_0) + 2\sqrt{r}\|\bz_0\|_2}{\mu}\right)
\end{align}
Hence, using \eqref{eq:DH0HhatCH0Hhat} under assumption \eqref{eq:conddelta2}, we have
\begin{align}
\cuL(\bH_0,\hbH)^{1/2} &\leq 3\sqrt{r}\max\left\{{(1+\sqrt{2})\sqrt{r}},{\sqrt{2}\kappa(\bH_0)}\right\}\left(8r\kappa(\bH_0)+ \frac{4\sigma_{\max}(\bH_0) + 2\sqrt{r}\|\bz_0\|_2}{\mu}\right)\frac{\delta}{\alpha}\\
&\leq 120\max\left\{1,\frac{\kappa(\bH_0)}{\sqrt{r}}\right\}\max\left\{r\kappa(\bH_0), \frac{\sigma_{\max}(\bH_0)+\sqrt{r}\|\bz_0\|_2}{\mu}\right\}
\frac{r\delta}{\alpha}.
\end{align}
Hence, for $C_*^{\prime\prime}$ as defined in the statement of the theorem, we get
\begin{align}
\cuL(\bH_0,\hbH)^{1/2}\le \frac{C_*^{\prime\prime}\,r}{\alpha} \max_{i\le n} \|\bZ_{i,\cdot}\|_2\, 
\end{align}
This completes the proof.

%
%
\section{Proof of Proposition \ref{propo:Subdiff}}
\label{app:Subdiff}

The proof follows immediately from the following two propositions.
\begin{proposition}
\label{prop:subgradient1}
Let $\bX\in \reals^{n\times d}$ and $D(\bx,\by) = \|\bx-\by\|_2^2$. Then the gradient of the function $\bu\mapsto \cuD(\bu,\bX)$ 
is given by
\begin{align}
\nabla_{\bu} \cuD(\bu,\bX) =  2(\bu - \bPi_{\conv(\bX)}(\bu)) \, . \label{eq:GradientD}
\end{align}
\end{proposition}
\begin{proof}
Note that $\cuD(\bu,\bX)$ is the solution of the following convex optimization problem.
\begin{align}
\label{eq:subgradopt}
\begin{split}
&\mbox{minimize}\quad \left\| \bu - \by\right\|_2^2,\\
&\mbox{subject to}\quad \by = \bX^\sT \bpi,\\
&\quad\quad\quad\quad\quad\;\bpi \geq 0,\\
&\quad\quad\quad\quad\quad\;\langle \bpi, \one \rangle = 1.
\end{split}
\end{align}
The Lagrangian for this problem is
\begin{align}
\mathcal L (\by,\bpi,\brho,\tilde\rho,\blambda) = \|\bu - \by\|_2^2 + \left\langle \brho,(\by-\bX^\sT\bpi)\right\rangle - \langle \blambda,\bpi\rangle + \tilde\rho(1-\langle \bpi,\one\rangle).
\end{align}
The KKT condition implies that at the minimizer $(\by^*,\bpi^*,\brho^*,\tilde\rho^*,\blambda^*)$, we have
\begin{align}
\frac{\partial \mathcal L}{\partial \by} = 0\, ,
\end{align}
and therefore
\begin{align}
\label{eq:kkt1}
\brho^* = 2(\bu - \by^*)
\end{align}
and the dual of the above optimization problem is
\begin{align}
\label{eq:subgraddual1}
\begin{split}
&\mbox{maximize}\quad  -\frac{1}{4}\|\brho\|_2^2 + \left\langle \brho, \bu\right\rangle + \tilde\rho,\\
&\mbox{subject to}\quad \blambda \geq 0,\\
&\quad\quad\quad\quad\quad\; \bX\brho + \tilde\rho\one+\blambda = 0.
\end{split}
\end{align}
Note that since \eqref{eq:subgradopt} is strictly feasible,
Slater condition holds and by strong duality the optimal value of \eqref{eq:subgraddual1} is equal to $f(\bu)$. Hence,
we have written $f(\bu)$ as pointwise supremum of functions. Therefore, subgradient 
of $f(\bu)$ can be achieved by taking the derivative of the objective function in \eqref{eq:subgraddual1} at the 
optimal solution (see Section 2.10 in \cite{mordukhovich2013easy}). Note that the derivative of this objective function at
the optimal solution is equal to $\brho^*= 2(\bu - \by^*) = 2(\bu - \bPi_{\conv(\bX)}(\bu))$ (where we used Eq.~\eqref{eq:kkt1}).
Since the dual optimum is unique (by strong convexity in $\brho$), the function $\bu\mapsto \cuD(\bu,\bX)$ is differentiable with gradient 
given by Eq.~(\ref{eq:GradientD}).
\end{proof}

\begin{proposition}
\label{prop:subgradient2}
Let $\bu\in \reals^{d}$ and $D(\bx,\by) = \|\bx-\by\|_2^2$, and assume that the rows of $\bH_0\in\reals^{r\times d}$ are affine independent.  
Then the function $\bH\mapsto \cuD(\bu,\bH)$ is differentiable at $\bH_0$ with gradient
\begin{align}
\nabla_{\bH}\cuD(\bu,\bH_0) = 
2\bpi_0(\bPi_{\conv(\bH_0)}(\bu) - \bu)^{\sT},\;\;\;\;\;\;\;
\bpi_0 = \arg\min_{\bpi \in \Delta^r} \left\|\bH_0^\sT\bpi - \bu\right\|^2_2\, . \label{eq:GradientFormula}
\end{align}
\end{proposition}

\begin{proof}
We will denote by $\bG$ the right hand side of Eq.~(\ref{eq:GradientFormula}). 
For $\bV\in\reals^{r\times d}$, we have
\begin{align}
\cuD(\bu,\bH_0+\bV) = \min_{\bpi \in \Delta^r} \left\|(\bH_0+\bV)^\sT\bpi - \bu\right\|^2_2\, .
\end{align}
Note that $(\bH_0+\bV)$ has affinely independent rows for $\bV$ in a neighborhood of $\bzero$, and hence
has a unique minimizer there, that we will denote by $\bpi_{\bV}$. By optimality of $\bpi_{\bV}$, we have
\begin{align}
\cuD(\bu,\bH_0+\bV)  -\cuD(\bu,\bH_0)& =  \left\|(\bH_0+\bV)^\sT\bpi_{\bV} - \bu\right\|^2_2-\left\|(\bH_0+\bV)^\sT\bpi_0 - \bu\right\|^2_2\\
& \le \left\|(\bH_0+\bV)^\sT\bpi_0 - \bu\right\|^2_2-\left\|(\bH_0+\bV)^\sT\bpi_0 - \bu\right\|^2_2\\
& = \<\bG,\bV\>+\|\bV\bpi_0\|_2^2.
\end{align}
On the other hand, by optimality of $\bpi_0$, 
\begin{align}
\cuD(\bu,\bH_0+\bV)  -\cuD(\bu,\bH_0) & \ge  \left\|(\bH_0+\bV)^\sT\bpi_{\bV} - \bu\right\|^2_2-\left\|(\bH_0+\bV)^\sT\bpi_\bV - \bu\right\|^2_2\\
&= \<2\bpi_\bV(\bPi_{\conv(\bH_0)}(\bu) - \bu)^{\sT},\bV\>+ +\|\bV\bpi_\bV\|_2^2\\
& =  \<\bG,\bV\>+2\<(\bpi_\bV-\bpi_0) (\bPi_{\conv(\bH_0)}(\bu) - \bu)^{\sT},\bV\>+\|\bV\bpi_\bV\|_2^2\, .
\end{align}
Letting $R(\bV) = |\cuD(\bu,\bH_0+\bV)  -\cuD(\bu,\bH_0)-\<\bG,\bV\>|$ denote the residual, we get
\begin{align}
\frac{R(\bV)}{\|\bV\|_F}\le \|\bPi_{\conv(\bH_0)}(\bu) - \bu\|_2\|\bpi_{\bV}-\bpi_0\|_2 +\|\bV\|_F(\|\bpi_\bV\|_2+\|\bpi_{0}\|_2)\, .
\end{align}
Note that $\bpi_{\bV}$ must converge to $\bpi_0$ as $\bV\to 0$ because $\bpi_0$ is the unique minimizer 
for $\bV=\bzero$. Hence we get $R(\bV)/\|\bV\|_F\to 0$ as $\|\bV\|_F\to 0$, which proves our claim.
\end{proof}
%
%
\section{Proof of Proposition \ref{propo:PALM}}
\label{app:PALM}

We use the results of \cite{bolte2014proximal} to prove Proposition \ref{propo:PALM}. 
We refer the reader to \cite{bolte2014proximal} for the definitions of the technical terms in this section.
First, consider the function 
\begin{align}\label{eq:f(H)definition}
f(\bH) = \lambda\cuD(\bH,\bX).
\end{align}
Note that using the main theorem of polytope theory (Theorem 1.1 in \cite{ziegler2012lectures}), we can write
\begin{align}
\conv(\bX) = \left\{\bx \in \reals^d \,|\, \langle \ba_i,\bx\rangle \leq b_i \;\; \text{for} \;\; 1\leq i\leq m\right\}
\end{align} 
for some $\ba_i \in \reals^d$, $b_i\in \reals$ and a finite $m$. Hence, using the definition
of the semi-algebraic sets (see Definition 5 in \cite{bolte2014proximal}), the set $\conv(\bX)$ is semi-algebraic.
Therefore, the function $f(\bH)$ which is proportional to the sum of squared $\ell_2$ distances of the rows of $\bH$ from a 
semi-algebraic set, is a semi-algebraic function (See Appendix in \cite{bolte2014proximal}). Further, the function 
\begin{align}\label{eq:g(W)definition}
g(\bW) = \sum_{i=1}^n \Ind\left(\bw_i\in \Delta^r\right)
\end{align}
is the sum of indicator functions of semi-algebraic sets
(Note that using the same argument used for $\conv(\bX)$, $\Delta^r$ is semi-algebraic). Therefore, 
the function $g$ is semi-algebraic (See Appendix in \cite{bolte2014proximal}). In addition,
the function 
\begin{align}\label{eq:h(H,W)definition}
h(\bH,\bW)  = \left\|\bX - \bW\bH\right\|_F^2
\end{align}
is a polynomial. Hence, it is semi-algebraic. Therefore, we deduce that the function
\begin{align}\label{eq:Psi(H,W)definition}
\Psi (\bH,\bW) = f(\bH) + g(\bW) + h(\bH,\bW)
\end{align}
is semi-algebraic. In addition, since $\Delta^r$ is closed, $\Psi$ is proper and lower semi-continuous. Therefore, $\Psi(\bH,\bW)$ 
is a KL function (See Theorem 3 in \cite{bolte2014proximal}).

Now, we will show that the Assumptions 1,2 in \cite{bolte2014proximal} hold for our algorithm. First, note that since $\Delta^r$ is closed,
the functions $f(\bH)$ and $g(\bW)$ are proper and lower semi-continuous. 
Further, $f(\bH)\ge 0,\, g(\bW)\ge 0,\, h(\bH,\bW)\ge 0$
for all $\bH\in \reals^{r\times d},\, \bW\in \reals^{n\times r}$. In addition, the function $h(\bH,\bW)$ is 
$C^2$. Therefore, it is Lipschitz continuous over the bounded subsets of $\reals^{r\times d}\times\reals^{n\times r}$. Also, 
the partial derivatives of $h(\bH,\bW)$ are
\begin{align}
&\nabla_\bH h(\bH,\bW) = 2\bW^\sT(\bW\bH-\bX),\\
&\nabla_\bW h(\bH,\bW) = 2(\bW\bH - \bX)\bH^\sT.
\end{align}
It can be seen that for any fixed $\bW$, the function $\bH \mapsto \nabla_\bH h(\bH,\bW)$ is Lipschitz continuous with moduli 
$L_1(\bW) = 2\|\bW^\sT\bW\|_F$. 
Similarly, for any fixed $\bH$, 
the function $\bW \mapsto \nabla_\bW h(\bH,\bW)$ is 
Lipschitz continuous with moduli $L_2(\bH) = 2\|\bH\bH^\sT\|_F$. Note that since in each iteration 
of the algorithm the rows of $\bW^k$ are in $\Delta^r$. Hence, 
\begin{align}
\inf\left\{L_1(\bW^k):k\in \mathbb{N}\right\}\ge \lambda_1^-\,\quad\quad \sup\left\{L_1(\bW^k):k\in \mathbb{N}\right\}\le \lambda_1^+\,
\end{align}
for some some positive constants $\lambda_1^-,\, \lambda_1^+$.
In addition, note that because the PALM algorithm is a descent algorithm,
i.e., $\Psi(\bH^{k}, \bW^{k}) \leq \Psi(\bH^{k-1}, \bW^{k-1})$ for $k \in \mathbb N$, and since $f(\bH)\to \infty$ as $\|\bH\|_F\to \infty$,
the value of $L_2(\bH^k) = \|\bH^k\bH^{k^\sT}\|_F$ remains bounded in every iteration. Finally, note that by taking 
$\gamma_2^k > \max\left\{\left\|\bH^{k+1}\bH^{k+1^\sT}\right\|_F,\eps\right\}$ for some constant $\eps>0$, we make 
sure that the steps in the PALM algorithm remain well defined (See Remark 3(iii) in \cite{bolte2014proximal}). Hence, we have shown 
that the assumptions of Theorem 1 in \cite{bolte2014proximal} hold. Therefore, using this theorem, the 
sequence $\left\{\bH^k,\, \bW^k\right\}_{k\in \mathbb N}$ generated by the iterations in \eqref{eq:PALMITER1} - \eqref{eq:PALMITER3}
has a finite length and it converges to a stationary point $\left(\bH^*,\bW^*\right)$ of $\Psi$.
%
\section{Other optimization algorithms}
\label{app:algo}

Apart from the proximal alternating linearized minimization discussed in Section \ref{sec:PALM},
we experimented with two other algorithms, obtaining comparable results. For the sake of completeness, we describe these algorithms here.

\subsection{Stochastic gradient descent}
\label{sec:Proximal}

Using any of the initializations discussed in Section \ref{sec:Initialization} we iterate
\begin{align}
\bH^{(t+1)} = \bH^{(0)} - \gamma_t \bG^{(t)}\, .
\end{align}
The step size $\gamma_t$ is selected by backtracking line search. Ideally, the direction $\bG^{(t)}$ can be taken to be equal to 
$\nabla\cuR_{\lambda}(\bH^{(t)})$. However, for large datasets this is computationally impractical, since it requires to compute the projection of each data
point onto the set $\conv(\bH^{(t)})$. In order to reduce the complexity of the direction calculation, we estimate this sum by subsampling.
Namely, we draw a uniformly random set $S_t\subseteq [n]$ of fixed size $|S_t|=s\le n$, and compute
\begin{align}
\bG^{(t)} &=  \frac{2n}{|S_t|}\sum_{i\in S_t} \balpha_{i}^*\left(\bPi_{\conv(\bH)}\left(\bx_{i}\right)-\bx_{i}\right) +2\lambda\left(\bH - \bPi_{\conv(\bX)}\left(\bH\right)\right) \,,\\
\balpha_{i}^{*} &= \arg\min_{\balpha \in \Delta^r} \left\| \bH^\sT \balpha - \bx_{i}^\sT\right\|_2\, .
\end{align}

\subsection{Alternating minimization}

This approach generalizes the original algorithm of \cite{cutler1994archetypal}. 
We rewrite the objective as a function of $\bW = (bw_i)_{i\le n}$, $\bw_i\in\Delta^r$,$\bH = (\bh_i)_{i\le r}$, $\bh_i\in\reals^d$ and
$\bA = (\balpha_{\ell})_{\ell\le r}$, $\balpha_{\ell}\in\Delta^n$
\begin{align}
\cuR_{\lambda}(\bH) &= \min_{\bW,\bA}F(\bH,\bW,\bA)\, ,\\ 
F(\bH,\bW,\bA) & = \sum_{i=1}^n\Big\|\bx_i-\sum_{\ell=1}^r w_{i\ell}\bh_{\ell}\Big\|_2^3+\lambda
\sum_{\ell=1}^r\Big\|\bh_{\ell}-\sum_{i=1}^n\alpha_{\ell,i}\bx_i\Big\|^2_2\, .
\end{align}
The algorithm alternates between minimizing with respect to the weights $(\bw_i)_{i\le n}$ (this can be done independently across $i\in\{1,\dots,n\}$) and minimizing
over $(\bh_{\ell},\balpha_{\ell})$, which is done sequentially by cycling over $\ell\in\{1,\dots,r\}$. Minimization over $\bw_i$ can be performed by solving a non-negative  
least squares problem. 
As shown in  \cite{cutler1994archetypal}, minimization over $(\bh_{\ell},\balpha_{\ell})$ is also equivalent to non-negative least squares. Indeed, by a simple calculation
\begin{align}
F(\bH,\bW,\bA) & = w^{\rm tot}_{\ell}\big\|\bh_{\ell}-\bv_{\ell}\big\|_2^2 + \lambda\Big\|\bh_{\ell}-\sum_{i=1}^n\alpha_{\ell,i}\bx_i\Big\|^2_2+\widetilde{F}(\bH,\bW,\bA)\\
& = f_{\ell}(\bh_{\ell},\balpha_{\ell};\bH_{\neq \ell},\bW,\bA)+\widetilde{F}(\bH,\bW,\bA)\, .
\end{align}
where $\bH_{\neq \ell} = (\bh_i)_{i\neq \ell, i\le r}$, $\widetilde{F}(\bH,\bW,\bA)$ does not depend on $(\bh_{\ell},\balpha_{\ell})$, and we defined
\begin{align}
w^{\rm tot}_{\ell} & \equiv \sum_{i=1}^nw^2_{i\ell}\, ,\\
\bv_{\ell} & \equiv \frac{1}{w^{\rm tot}_{\ell}}\,\sum_{i=1}^nw_{i,\ell}\left\{\bx_i-\sum_{j\neq \ell, j\le r} w_{ij}\bh_j\right\}\, .
\end{align}
It is therefore sufficient to minimize $f_{\ell}(\bh_{\ell},\balpha_{\ell};\bH_{\neq \ell},\bW,\bA)$ with respect to its first two arguments,
which is equivalent to a non-negarive least squares problem. This can be seen by minimizing $f_{\ell}(\cdots )$ explicitly with respect to $\bh_{\ell}$ and
writing the resulting objective function.

The pseudocode for this algorithm is given below. 
\begin{center}
	\begin{tabular}{ll}
	\hline
	\vspace{-0.35cm}\\
	\multicolumn{2}{l}{ {\sc Alternating minimization}}\\
	\hline
	\vspace{-0.35cm}\\
	\multicolumn{2}{l}{ {\bf Input :}  Data $\{\bx_i\}_{i\le n}$,  $\bx_i\in\reals^d$; integer $r$; initial archetypes $\{\bh_{\ell}^{(0)}\}_{1\le \ell\le r}$; number of iterations $T$;} \\
	\multicolumn{2}{l}{ {\bf Output :} Archetype estimates $\{\bh_{\ell}^{(T)}\}_{1\le \ell\le r}$;}\\
        1: & For $\ell\in\{1,\dots,r\}$:\\
        2: &\phantom{aa} Set $\balpha^{(0)}_{\ell} = \arg\min _{\balpha\in\Delta^n}\|\bh^{(0)}_\ell - \bX\balpha_{\ell}\|_2$;\\
	3: & For $t\in \{1,\dots, T\}$:\\
        4: &\phantom{aa}  Set $\bW^{t} = \arg\min_{\bW}F(\bH^{t-1},\bW,\bA^{t-1})$\\
        5: &\phantom{aa}  For $\ell\in\{1,\dots,r\}$:\\
        6: &\phantom{aa} Set $\bh^{(t)}_{\ell},\balpha^{(t)}_{\ell} = \arg\min_{\bh_{\ell},\balpha_{\ell}}f_{\ell}(\bh_{\ell},\balpha_{\ell};\bH_{<\ell}^{t},\bH_{>\ell}^{t-1},\bW^{t},\bA_{<\ell}^{t},\bA_{>\ell}^{t-1})$;\\
	7: & End For;\\
	8: & Return $\{\hbh_{\ell}^{(T)}\}_{1\le \ell\le r}$\\
	\vspace{-0.35cm}\\
	\hline
	\end{tabular}
\end{center}

Here $\bH_{< \ell} = (\bh_i)_{ i< \ell}$, $\bH_{> \ell} = (\bh_i)_{\ell< i\le r}$, and similarly for $\bA$.

\end{document}